\documentclass[11pt,letterpaper]{article}

\usepackage{amsmath, amssymb, amsthm, bm, bbm}
\usepackage[margin=1.1in]{geometry}
\usepackage{authblk}

\usepackage[utf8]{inputenc} 
\usepackage{hyperref}       
\usepackage{url}            
\usepackage{booktabs}       
\usepackage{microtype}      
\usepackage{xcolor}    
\usepackage{hyperref}
\usepackage{enumerate}
\hypersetup{colorlinks=true, linkcolor=blue, citecolor=blue, filecolor=blue, urlcolor=blue}
\usepackage[square]{natbib}
\usepackage{graphicx}
\usepackage{wrapfig}
\usepackage{multirow}
\usepackage{algorithm}
\usepackage{algorithmic}
\usepackage{xargs}
\usepackage[most]{tcolorbox}
\usepackage{subcaption}
\usepackage{enumitem}

\definecolor{phaseblue}{HTML}{E3EDF8}
\definecolor{phasegreen}{HTML}{DDEFE7}
\definecolor{phasepink}{HTML}{F8E5D9}
\definecolor{phasepurple}{HTML}{EAE3F3}

\newtcolorbox{phasebox}[2]{%
  enhanced,
  frame hidden,
  boxrule=0pt,
  sharp corners,
  colback=#1,
  left=1mm,
  right=1mm,
  top=0.7mm,
  bottom=0.7mm,
  boxsep=0pt,
  width=\dimexpr\linewidth-2.2cm\relax,
  before skip=2pt,
  after skip=2pt,
  overlay={
    \node[
      anchor=west,
      font=\footnotesize\bfseries,
      align=center,
      text width=1.8cm
    ] at ([xshift=6pt]frame.east) {#2};
  }
}

\numberwithin{equation}{section}

\theoremstyle{plain}
\newtheorem{theorem}{Theorem}[section]
\newtheorem{lemma}[theorem]{Lemma}

\newtheorem{corollary}[theorem]{Corollary}

\theoremstyle{definition}
\newtheorem{definition}{Definition}
\newtheorem{assumption}{Assumption}
\newtheorem{remark}{Remark}[section]

\newcommand{\cD}{\mathcal{D}}
\newcommand{\cG}{\mathcal{G}}

\newcommand{\cX}{\mathcal{X}}
\newcommand{\cY}{\mathcal{Y}}

\newcommand{\E}{\mathbb{E}}
\newcommand{\N}{\mathbb{N}}
\renewcommand{\P}{\mathbb{P}}
\newcommand{\R}{\mathbb{R}}

\newcommand{\one}{\mathbbm{1}}

\renewcommand{\d}{\mathrm{d}}
\newcommand{\test}{\mathrm{test}}
\newcommand{\rmN}{\mathrm{N}}
\newcommand{\train}{\mathrm{train}}
\renewcommand{\cal}{\mathrm{cal}}
\newcommand{\oracle}{\mathrm{oracle}}
\newcommand{\eff}{\mathrm{eff}}
\newcommand{\TV}{\mathrm{TV}}

\usepackage[nameinlink,noabbrev]{cleveref}
\crefname{assumption}{assumption}{assumptions}

\newcommand{\papertitle}{Multi-source conformal prediction: leveraging heterogeneity via localization}
\newcommand{\paperauthorA}{Rohan Hore}
\newcommand{\affilOne}{Department of Statistics, Stanford University}
\newcommand{\paperauthorB}{Anirban Chatterjee}
\newcommand{\affilTwo}{Department of Mathematics and Statistics, Boston University}
\newcommand{\paperauthorC}{Sayantan Choudhury}
\newcommand{\affilThree}{Department of Computing and Mathematical Sciences, MBZUAI}
\newcommand{\keywordslist}{Multi-source predictive inference, covariate shift, conformal prediction, local coverage, multi-environment prediction}

\title{\papertitle}
\author[1]{\paperauthorA}
\author[2]{\paperauthorB}
\author[3]{\paperauthorC\textsuperscript{\(\dagger\)}}
\affil[1]{\affilOne}
\affil[2]{\affilTwo}
\affil[3]{\affilThree}

\date{\today}

\newcommand{\paperabstract}[1]{%
  \begin{abstract}
    #1
  \end{abstract}
}
\newcommand{\paperkeywords}[1]{%
  \noindent\textbf{Keywords: } #1
}

\usepackage[nameinlink,noabbrev]{cleveref}

\begin{document}
\maketitle
\begingroup
\renewcommand{\thefootnote}{\fnsymbol{footnote}}
\footnotetext[1]{These authors contributed equally to this work.}
\footnotetext[2]{Now at Amazon, India.}
\endgroup
\paperabstract{
Many modern prediction tasks involve data from multiple heterogeneous sources, while the test distribution may differ substantially from any individual source. Although heterogeneity poses challenges, it also offers an opportunity: different sources may provide complementary information, with some regions of the feature space better represented in one source than another. We propose Multi-Source Randomly Localized Conformal Prediction (MS-RLCP), which builds on the local coverage properties of randomly localized conformal prediction (RLCP) \citep{hore2025conformal} and extends it to multiple sources through data-adaptive source selection. Under the widely adopted assumption of a shared response distribution conditional on the features across sources and the test population, we establish finite-sample coverage bounds using an interpretable notion of \emph{envelope distribution} that captures their aggregate feature-space representation. Our analysis allows the test feature distribution to be absolutely continuous with respect to the envelope, extending beyond mixtures of source distributions. Under additional regularity conditions, we also establish asymptotic test-conditional coverage. Simulations and real-world experiments demonstrate the effectiveness of MS-RLCP across varying levels of data heterogeneity.
}

\paperkeywords{\keywordslist}

\section{Introduction}

Recent advances in machine learning methods have led to their increased adoption in safety-critical applications, including autonomous driving \citep{bojarski2016end}, medical diagnosis \citep{kompa2021second, bhatt2021uncertainty}, and financial risk assessment \citep{mashrur2020machine}. This has made reliable uncertainty quantification increasingly important in these settings (see \cite{guo2017calibration, he2026survey, ovadia2019can} and references therein). Conformal prediction \citep{vovk2005algorithmic} provides a simple yet powerful distribution-free framework for constructing prediction sets with finite-sample coverage guarantees. These guarantees typically rely on exchangeability, an assumption satisfied when training and test data are drawn i.i.d.\ from the same population.

However, many applications rely on multiple heterogeneous sources collecting data under different conditions or protocols \citep{crammer2008learning,mansour2008domain,zhao2018adversarial}. For instance, in healthcare, predicting a new patient's response to a drug may involve clinical data from hospitals serving different patient populations \citep{ali2022federated,jochems2016distributed}. Similarly, sensor networks combine measurements from devices operating under different environmental conditions \citep{mcmahan2017communication,zhao2025blockchain}. Such differences across source populations can violate the exchangeability assumption underlying standard conformal prediction.

While the challenges posed by heterogeneity are well recognized in the literature, the opportunity to combine complementary information from different sources remains less explored, particularly when some sub-populations are better represented in one source than another. This motivates our central question: \emph{how can we bring together such local information to construct efficient conformal prediction sets?} We describe our formal setup and goal below.

\subsection{Problem setup}

Let $\N$ denote the set of natural numbers and consider $K \in \mathbb{N}$ heterogeneous data sources. For each $k \in [K]$, let $\mathcal{D}_k = \{(Y_{i,k}, X_{i,k}) : 1 \le i \le n_k\}$ consist of $n_k \in \mathbb{N}$ independent samples from a distribution $P_k$. We also assume that the datasets $\mathcal{D}_k$ are independent across $k$. Here, $X_{i,k} \in \mathcal{X}$ is a feature vector and $Y_{i,k} \in \mathcal{Y}$ is the corresponding response. Write $\mathcal{D} = \bigcup_{k=1}^{K} \mathcal{D}_k$ for the combined dataset and $n = |\mathcal{D}|$ for its total sample size.

Let $2^{\mathcal{Y}}$ denote the power set of $\mathcal{Y}$. The aim of this work is to construct a prediction set $\hat{C}_n : \mathcal{X} \to 2^{\mathcal{Y}}$ that satisfies
\begin{equation}\label{eq:conformal_coverage}
\P_{\mathcal{D} \times P_{\text{test}}}\!\left( Y_{n+1} \in \hat{C}_n(X_{n+1}) \right) \ge 1 - \alpha
\qquad \text{for all } P_{\text{test}}.
\end{equation}
Here, the probability is taken over the randomness of the observed data $\mathcal{D}$ and a test sample $(Y_{n+1}, X_{n+1})$ drawn from $P_{\text{test}}$ independently of $\mathcal{D}$.

While~\eqref{eq:conformal_coverage} may at first glance seem a meaningful goal, without assumptions on the source and test distributions, nontrivial prediction sets cannot generally satisfy \eqref{eq:conformal_coverage}, even with a single source ($K=1$). This is closely tied to the difficulty of achieving distribution-free conditional coverage \citep{vovk2012conditional,foygel2021limits, lei2014distribution}; see \citet[Section~1]{hore2025conformal} for a formal connection between the two.

A natural starting point is therefore to adopt the widely used covariate shift assumption \citep{tibshirani2019conformal}: $P_k = P_{Y \mid X} \times P_{k,X}$ for all $k \in [K]$ and $P_{\text{test}} = P_{Y \mid X} \times P_{\text{test},X}$. In other words, the marginal feature distributions may differ, while the conditional distribution of the response given the features, $P_{Y\mid X}$, is shared across sources and the test environment.

Even under this covariate shift assumption, exact coverage guarantees remain difficult to attain in practice, including when $K=1$. For instance, the weighted conformal method of \citet{tibshirani2019conformal} provides an exact coverage guarantee when the true shift function $\frac{\mathsf{d} P_{\text{test},X}}{\mathsf{d} P_{1,X}}$ is known. In practice, this function must typically be estimated, which can be challenging when the distributions are supported on high-dimensional spaces. 

To address this difficulty, a rich literature seeks approximate coverage guarantees for the setting $K=1$ (see \cite{gibbs2021adaptive, podkopaev2021distribution, liu2024multi} and references therein). In line with these works, we also work under covariate shift and aim to construct prediction sets from heterogeneous multi-source data with \emph{approximate} coverage: the guarantee in \eqref{eq:conformal_coverage} holds up to a provably small deviation from the nominal level $1-\alpha$ for any finite $K\ge 1$.

\subsection{Related work}

Given the widespread relevance of multi-source predictive inference in practical problems, a substantial body of literature has studied it from different perspectives. Below, we provide a selective review of existing work.

\paragraph{Federated conformal prediction.}
A prominent line of work arises from the federated learning literature, where methods combine information across heterogeneous sources \citep{lu2023federated,plassier2024efficient,zhu2024federated}. For example, \citet{lu2023federated} establish coverage under partial exchangeability, while \citet{liu2024multi} develop target-population inference using influence-function-based estimation and adaptive source weighting. Some methods further reduce communication and data sharing; notably, \citet{humbert2023one} propose a quantile-of-quantiles procedure requiring only one round of communication. These considerations are particularly relevant in healthcare, where sharing individual observations can be restricted \citep{sun2026unlocking,adnan2022federated}. A common formulation in this literature takes the test population to be a mixture of source populations, as in \citet{lu2023federated}; guarantees under this formulation do not directly extend to test distributions outside the mixture class.

\paragraph{Aggregation across sources.}
Since conformal prediction can be viewed through the lens of testing, a natural approach is to construct source-specific conformal prediction sets or $p$-values and then aggregate them. \citet{yang2026multi} develop max-based aggregation of conformal $p$-values, while \citet{ying2024informativeness} study aggregation of source-specific weighted conformal prediction sets. However, the efficiency of aggregation depends on score design: scores that perform well separately need not yield efficient prediction sets after aggregation.

\paragraph{Group-conditional coverage.}
When different sources are interpreted as groups or classes, multi-source predictive inference is closely related to group-conditional coverage. Several works study coverage guarantees within specified groups or classes \citep{romano2020malice,jung2022batch,ding2023class,bairaktari2025kandinsky}. Coverage of at least $1-\alpha$ within every group also ensures coverage under any mixture of those group distributions. A related formulation considers covariate shift determined by changes in group proportions \citep{bhattacharyya2026group}. However, guarantees are limited to test distributions falling within the mixture class.

\paragraph{Hierarchical conformal prediction.}
Distributed data are also studied through hierarchical models, in which source distributions are first drawn from a common higher-level population and observations within each source are then sampled conditionally on its distribution. \citet{dunn2023distribution,lee2023distribution,duchi2025predictive} develop conformal procedures for hierarchical or multi-environment data. Guarantees based on this structure require similar hierarchical structure for test distribution.

\paragraph{Coverage under distribution shift.}
More broadly, since we allow $P_{\text{test}}$ to differ from the source populations, our goal is closely connected to robust conformal prediction under distribution shift. In the single-source setting, weighted conformal prediction accounts for covariate shift through likelihood-ratio weighting when the shift is known \citep{tibshirani2019conformal}. \citet{cauchois2024robust,xu2025wasserstein,aolaritei2026conformal} further study robustness to general distributional shifts through different notions of distributional discrepancy.

\paragraph{Conformal prediction with local coverage guarantees.}
When distribution shift is restricted to covariate shift, constructing robust prediction sets is naturally linked to achieving good local coverage. Despite the challenges of distribution-free test-conditional coverage, several works develop theoretically grounded conformal methods with local or conditional coverage guarantees under appropriate assumptions or relaxations \citep{chernozhukov2021distributional,gibbs2025conformal,guan2023localized,hore2025conformal}.

\subsection{Our approach}\label{sec:our_approach}

We approach the multi-source prediction problem by constructing a conformal prediction set at each source and selecting among these sets using a data-adaptive rule. The intuition, supported by empirical evidence, is that local coverage of conformal prediction sets is typically better in regions that are more strongly represented in the data. With multiple sources available, we can match a test point to a source with stronger local representation and use its conformal prediction set to help improve coverage at that test point.

To implement this idea, we first revisit RLCP framework of \cite{hore2025conformal}, which constructs prediction sets with provable local coverage guarantees in the single-source setting. We then propose \underline{M}ulti-\underline{S}ource \underline{R}andomly \underline{L}ocalized \underline{C}onformal \underline{P}rediction (\texttt{MS-RLCP}), which forms source-specific RLCP prediction sets for each of the $K$ sources and adaptively selects one at test time. Integrating data-adaptive source selection into RLCP yields the following useful properties:
\begin{itemize}
    \item \textbf{Local training without data sharing.}
    Each source-specific RLCP prediction set can use its own conformal score, allowing each source to train its prediction model independently without sharing training data. This reduces the communication burden associated with centralized training. Moreover, \texttt{MS-RLCP} can serve as a wrapper around ready-to-use source-specific models, without requiring any model aggregation or global calibration.

    \item \textbf{Meaningful coverage beyond mixtures of sources.}
    We introduce an interpretable notion of \emph{envelope distribution} that captures the aggregate representation of the feature space across sources. We provide finite-sample upper bounds on coverage under $P_{\text{test}}$ that are absolutely continuous with respect to this envelope. In particular, the test covariate distribution need not be a mixture of the source distributions; it may only be absolutely continuous with respect to any such mixture.

    \item \textbf{Asymptotic test-conditional coverage.}
    Beyond marginal coverage, we characterize coverage conditional on the test feature and establish that \texttt{MS-RLCP} inherits the asymptotic conditional coverage guarantees of RLCP, despite data-adaptive source selection.
\end{itemize}

Our approach is related to data-adaptive selection among conformal prediction sets, as studied by \citet{yang2025selection,liang2026conformal,wang2026localized}. These works, however, typically focus on choosing among prediction models in a single-source setting to improve efficiency while preserving validity. In our setting, the test distribution may differ from every source, so relying on any single source may lead to coverage far below the nominal level. Establishing meaningful coverage guarantees through source
selection is therefore the primary challenge.
\section{Multi-Source Randomly Localized Conformal Prediction}
In this section, we formally introduce the Multi-Source Randomly Localized Conformal Prediction (\texttt{MS-RLCP}) algorithm for constructing reliable prediction sets from heterogeneous data sources. First, in Section~\ref{sec:warmup}, we illustrate its guiding principles in a simplified setting with known source alignment. We then review RLCP, the building block of our approach, in Section~\ref{sec:local_coverage_RLCP}, before presenting the full \texttt{MS-RLCP} method in Section~\ref{sec:MS-RLCP}.

\subsection{A Structured Multi-Source Setting with Known Alignment}\label{sec:warmup}

To illustrate the key ideas underlying our approach, we begin with a simple multi-source setting in which constructing a prediction set reduces to selecting the source aligned with a given test point. Assume a known partition of the feature space $\mathcal{X}=\bigsqcup_{k=1}^{K}\mathcal{X}_k$ such that, for every $k\in[K]$, the test feature distribution matches the $k$-th source feature distribution when restricted to $\mathcal{X}_k$. Formally,
\begin{equation}\label{eq:simple_test_dist}
P_{\text{test},X \mid X\in \mathcal{X}_k}
= P_{k,X \mid X\in \mathcal{X}_k},
\qquad \text{for all } k\in[K].
\end{equation}

Equivalently, the test feature distribution is a mixture of the conditional distributions
$P_{k,X\mid X\in\mathcal{X}_k}$, $k\in[K]$. In this setting, the coverage guarantee in~\eqref{eq:conformal_coverage} can be achieved through a simple extension of split conformal prediction. For each source $k\in[K]$, construct a source-specific prediction set $\hat{C}_{n_k,k}:\mathcal{X}\to 2^{\mathcal{Y}}$, where $n_k = |\mathcal{D}_k|$, by applying split conformal prediction \citep{vovk2005algorithmic, papadopoulos2002inductive} to the samples in $\mathcal{D}_k$ whose features lie in $\mathcal{X}_k$. This prediction set satisfies the following conditional coverage guarantee (see~\cite{lei2012distribution}):
\begin{equation}\label{eq:cond_cover}
\mathbb{P}_{P_k\times\mathcal{D}_k}\!\left(
Y_{n+1}\in \hat{C}_{n_k,k}(X_{n+1})
\;\middle|\;
X_{n+1}\in \mathcal{X}_k
\right)\ge 1-\alpha .
\end{equation}

That is, $\hat{C}_{n_k,k}$ provides valid coverage for a test sample drawn from $P_k$ conditional on $\{X_{n+1}\in \mathcal{X}_k\}$
We then define a global prediction rule $\hat{C}_n:\mathcal{X}\to 2^{\mathcal{Y}}$ by setting $\hat{C}_n(x)=\hat{C}_{n_k,k}(x)$ whenever $x\in\mathcal{X}_k$, for $k\in[K]$. Under covariate shift, combining the alignment condition in \eqref{eq:simple_test_dist} with the coverage guarantee in \eqref{eq:cond_cover} shows that $\hat C_n$ satisfies \eqref{eq:conformal_coverage} without an approximation error.

The guiding principles of this construction, which also underlies our general proposal, are as follows:
\begin{itemize}
\item Construct source-specific prediction sets with local coverage guarantees.
\item For each test point, identify the source with the most appropriate local representation in the feature space.
\end{itemize}

Intuitively, we select the source that best represents the neighborhood of a new test point and report its locally calibrated prediction set as the final prediction set. In the structured setting above, the known partition determines this selection; in our general method, we select the source using a data-adaptive rule. The formal details are given in Section~\ref{sec:MS-RLCP}.

\subsection{RLCP and local coverage}\label{sec:local_coverage_RLCP}

We briefly review Randomly Localized Conformal Prediction (RLCP) from \cite{hore2025conformal}, which constructs prediction sets with provable local coverage guarantees and forms a key component of our \texttt{MS-RLCP} method.

Consider the single-source setting. With a slight abuse of notation, let
$\cD=\{(Y_i,X_i):1\le i\le n\}$
denote observations generated independently from an unknown distribution $P$. Let $(Y_{n+1},X_{n+1})$ denote an independent test sample from the same distribution $P$.

For a user-specified kernel $H(\cdot,\cdot)$ (assumed symmetric in its arguments, with $H(x,\cdot)$ a density for all $x\in\cX$), generate a perturbed test feature $\tilde X_{n+1}\sim H(X_{n+1},\cdot)$ conditional on $X_{n+1}$. Using the original test feature $X_{n+1}$ and the perturbed feature $\tilde X_{n+1}$, RLCP aims to construct a prediction set $\hat C_n(X_{n+1},\tilde X_{n+1})$ satisfying
\begin{align}
\P\!\left(Y_{n+1}\in \hat C_n(X_{n+1},\tilde X_{n+1}) \mid \tilde X_{n+1}\right)\ge 1-\alpha.
\label{eq:rlcp_target}
\end{align}
For a suitably localized kernel $H$, we expect $\tilde X_{n+1}\approx X_{n+1}$. The notion of conditional coverage in~\eqref{eq:rlcp_target} can be interpreted as coverage conditional on a random
neighborhood of $X_{n+1}$.

For $\cX=\R^d$ with $d\ge 1$, natural choices of kernel $H$ include the Gaussian and box kernels:
\begin{eqnarray}\label{eq:gauss_box_kernels}
    H(x,y) & = & \frac{1}{(2\pi)^{d/2}h^d}\exp\!\left(-\frac{\|x-y\|_2^2}{2h^2}\right), \notag\\
    H(x,y) & = & \frac{1}{V_d h^d}\mathbf{1}\{\|x-y\|_2\le h\}, \qquad x,y\in\R^d,
\end{eqnarray}
where $h>0$ is the bandwidth and $V_d$ denotes the volume of the unit ball in $\R^d$. The Gaussian kernel corresponds to perturbing $X_{n+1}$ with Gaussian noise, i.e., generating $\tilde X_{n+1}=X_{n+1}+hZ$ with $Z\sim\mathcal{N}_d(0,I_d)$, while the box kernel corresponds to sampling $\tilde X_{n+1}$ uniformly from the Euclidean ball of radius $h$ centered at $X_{n+1}$.

To achieve the local coverage guarantee in \eqref{eq:rlcp_target}, RLCP adopts the weighted conformal approach originally proposed in \cite{tibshirani2019conformal}. To begin, define a nonconformity score $s:\cX\times\cY\to\R$, where larger values indicate greater nonconformity of a sample with respect to the observed data $\mathcal{D}$. In practice, a split-conformal approach is adopted, where $\cD$ is first partitioned into a training set
$\cD_{\text{train}}=\{(Y_i,X_i):1\le i\le n_{\text{train}}\}$
and a calibration set
$\cD_{\text{cal}}=\{(Y_i,X_i):n_{\text{train}}+1\le i\le n\}$. The score is learned using the training data $\cD_{\text{train}}$, for instance by fitting a predictive model and defining the score through its residuals. In a regression setting, one may use
$s(x,y)=|y-\hat f(x)|$, where $\hat f$ is the fitted predictor.

Given $s$, the calibration set $\cD_{\text{cal}}$ is used to determine a data-dependent threshold. Let $\delta_x$ denote the Dirac measure at $x$. We define the weighted empirical distribution
\begin{equation}\label{eq:weighted_empirical_dist}
\sum_{i=n_{\text{train}}+1}^{n} \!\!\!\!\!\tilde{w}_i \,\delta_{s(X_i,Y_i)} + \tilde{w}_{n+1}\,\delta_{+\infty}
\quad \text{with }\quad
\tilde{w}_i = \frac{H(X_i,\tilde X_{n+1})}{\sum_{j=n_{\text{train}}+1}^{n+1} H(X_j,\tilde X_{n+1})},
\quad \forall i\in [n+1]\setminus [n_{\text{train}}].
\end{equation}
The RLCP prediction set is then defined as
\begin{align}\label{eqn:def_RLCP}
\hat{C}_n^{\textnormal{RLCP}}(X_{n+1},\tilde X_{n+1};\cD_{\text{train}},\cD_{\text{cal}})
= \left\{y\in\cY : s(X_{n+1},y)\le \hat{q}_{1-\alpha}(X_{n+1},\tilde X_{n+1})\right\},
\end{align}
where $\hat{q}_{1-\alpha}(X_{n+1},\tilde X_{n+1})$ denotes the $(1-\alpha)$-quantile of the weighted empirical distribution in \eqref{eq:weighted_empirical_dist}. \citet[Proposition 1]{hore2025conformal} shows that $\hat C_n^{\mathrm{RLCP}}$ achieves the coverage guarantee in \eqref{eq:rlcp_target}. We restate this result below for later reference.

\begin{lemma}[Proposition 1, \cite{hore2025conformal}]\label{lemma:rlcp_hore_barber_validity}
    The RLCP prediction set constructed in \eqref{eqn:def_RLCP} satisfies
    \begin{align*}
        \P\left({Y_{n+1}\in \hat{C}_n^{\rm RLCP}(X_{n+1},\tilde X_{n+1};\cD_{\rm train},\cD_{\rm cal})\mid \tilde X_{n+1},\cD_{\rm train}}\right)\ge 1-\alpha.
    \end{align*}
\end{lemma}
We next introduce our main method, \texttt{MS-RLCP}, for reliable predictive coverage with heterogeneous data.

\subsection{Multi-Source RLCP}\label{sec:MS-RLCP}

We now formally present our main method, \texttt{MS-RLCP}. Following the intuition from Section~\ref{sec:warmup} and the outline in Section~\ref{sec:our_approach}, we combine the local coverage properties of source-specific RLCP prediction sets with a data-driven rule that selects the source with the strongest local representation near the test point.

\begin{algorithm}[t]
\caption{Multi-Source Randomly Localized Conformal Prediction (\texttt{MS-RLCP})}
\label{alg:ms_rlcp}
\small
\begin{algorithmic}[1]
\REQUIRE Source datasets $\{\cD_k\}_{k=1}^K$, test feature $X_{n+1}$, kernel $H$, target miscoverage level $\alpha$, and score-fitting routine $\mathsf{FitScore}$.
\begin{phasebox}{phaseblue}{Randomize \\ Test Feature}
\STATE Draw $\tilde X_{n+1} \sim H(X_{n+1},\cdot)$;
\end{phasebox}
\begin{phasebox}{phasegreen}{Train \& \\ Align Sources}
\FOR{$k=1,\dots,K$}
    \STATE Split $\mathcal D_k=\mathcal D_{k,\mathrm{train}}\sqcup \mathcal D_{k,\mathrm{cal}}$ and fit $s_k \leftarrow \mathsf{FitScore}(\mathcal D_{k,\mathrm{train}})$;
    \STATE Compute the local source alignment score
    $$
    \hat w_k
    \leftarrow
    \frac{1}{|\mathcal D_{k,\mathrm{train}}|}
    \sum_{(Y,X)\in \mathcal D_{k,\mathrm{train}}}
    H(X,\tilde X_{n+1});
    $$
\ENDFOR
\STATE Select $\hat k \leftarrow \arg\max_{k\in[K]} \hat w_k$;
\end{phasebox}
\begin{phasebox}{phasepink}{Local \\ Calibration}
\STATE For each $(Y_{i,\hat k},X_{i,\hat k})\in \mathcal D_{\hat k,\mathrm{cal}}$, set $S_{i,\hat k} \leftarrow s_{\hat k}(X_{i,\hat k},Y_{i,\hat k})$;
\STATE Set
$
Z \leftarrow
H(X_{n+1},\tilde X_{n+1})
+
\sum_{\mathcal D_{\hat k,\mathrm{cal}}}
H(X_{i,\hat k},\tilde X_{n+1});
$
\STATE Set $w_{i,\hat k} \leftarrow H(X_{i,\hat k},\tilde X_{n+1})/Z$ and $w_{n+1}\leftarrow H(X_{n+1},\tilde X_{n+1})/Z$;
\STATE Compute
$
\hat q \leftarrow
\mathrm{Quantile}_{1-\alpha}\!\left(
\sum_{\mathcal D_{\hat k,\mathrm{cal}}}
w_{i,\hat k}\delta_{S_{i,\hat k}}
+
w_{n+1}\delta_{+\infty}
\right);
$
\end{phasebox}
\begin{phasebox}{phasepurple}{Predict}
\STATE \textbf{Return}
$
\hat C_n(X_{n+1})
\leftarrow
\{y\in\mathcal Y: s_{\hat k}(X_{n+1},y)\le \hat q\}.
$
\end{phasebox}
\end{algorithmic}
\end{algorithm}

Consider a test sample $(Y_{n+1},X_{n+1})$ drawn from an unknown test distribution $P_{\text{test}} = P_{Y\mid X}\times P_{\text{test},X}$. Given a kernel $H$, generate the perturbed test feature conditional on $X_{n+1}$:
\begin{align}\label{eq:perturbed_X}
   \tilde X_{n+1}\sim H(X_{n+1},\cdot).
\end{align}

\paragraph{Step I: Source-Specific Prediction Sets with Local Coverage.}
To construct prediction sets with good local coverage, we apply the RLCP framework from Section~\ref{sec:local_coverage_RLCP} independently within each source. For each $k\in[K]$, we partition the data $\cD_k=\cD_{k,\train}\bigsqcup\cD_{k,\cal}$ into training and calibration subsets and construct an RLCP prediction set as in \eqref{eqn:def_RLCP}:
\begin{equation}\label{eq:k_source_pred_set}
\hat C_k\!\left(X_{n+1},\tilde X_{n+1}\right)
= \hat C_{n_k}^{\rm RLCP}\!\left(X_{n+1},\tilde X_{n+1};\cD_{k,\train},\cD_{k,\cal}\right).
\end{equation}
This is achieved by training a source-specific score function $s_k$ on $\cD_{k,\train}$ and calibrating it using the data-dependent quantile threshold computed from $\cD_{k,\cal}$.

\paragraph{Step II: Source Identification through Data-Dependent Alignment.}
Next, to complete the construction of the prediction set, we align the test observation with the most appropriate source. Specifically, for each source, we evaluate how likely the observed perturbed test point $\tilde X_{n+1}$ would be if the original feature were drawn from that source and then perturbed according to~\eqref{eq:perturbed_X}. To that end, for each $k\in [K]$ let $P_{k,\tilde X}$ be the distribution of $\tilde X$, where $\tilde X$ is obtained by first sampling $X\sim P_{k,X}$ and then generating $\tilde X\mid X\sim H(X,\cdot)$. Thus, $P_{k,\tilde X}$ admits the density
\begin{align}
\label{eq:pop_den_Xtilde_k}
\tilde x \;\mapsto\; \mathbb{E}_{X \sim P_{k,X}}\!\left[H(X,\tilde x)\right].
\end{align}
We select the source under which the perturbed test point has the largest density, giving the oracle source selection rule
\begin{equation}\label{eq:k_oracle}
k_{\mathrm{oracle}}(\tilde X_{n+1})
\;:=\;
\arg\max_{k \in [K]}
\mathbb{E}_{X \sim P_{k,X}}\!\left[H(X,\tilde X_{n+1})\right].
\end{equation}
Since the population density in \eqref{eq:pop_den_Xtilde_k} is unknown, we estimate it using the empirical mean based on the training samples $\mathcal D_{k,\mathrm{train}}$. This leads to the following data-dependent source selection rule:
\begin{equation}\label{eq:source_choice}
\hat{k}(\tilde X_{n+1})
=\arg\max_{k \in [K]}
\frac{1}{n_{k,\mathrm{train}}}
\sum_{i=1}^{n_{k,\mathrm{train}}}
H\!\left(X_{i,k},\tilde X_{n+1}\right),
\end{equation}
where $n_{k,\mathrm{train}} = |\mathcal D_{k,\mathrm{train}}|$. Ties in both selection rules can be resolved using a pre-fixed rule. With this choice of source, we report the \texttt{MS-RLCP} prediction set
\begin{equation}\label{eq:MSRLCP_pred_set}
    \hat C_n(X_{n+1}):=\hat C_{\hat k(\tilde X_{n+1})}\left(X_{n+1},\tilde X_{n+1}\right),
\end{equation}
where the prediction set on the right-hand side is defined in \eqref{eq:k_source_pred_set}. The resulting procedure is summarized in Algorithm~\ref{alg:ms_rlcp}.

\begin{remark}
A key feature of our framework is that we allow \emph{source-specific scores} in \textbf{Step I}. In particular, each source $k\in[K]$ is free to choose its score function $s_k$. This flexibility allows each source to draw on its own strengths and account for its particular data characteristics and sample diversity. In contrast to conformal approaches that aggregate data and learn a single score function from the pooled dataset \citep{lu2023federated, plassier2024efficient, zhu2024federated}, our approach requires no data sharing during model training, substantially reducing the corresponding communication burdens.
\end{remark}
In Section~\ref{sec:msrlcp_theory_main}, we establish the validity of the \texttt{MS-RLCP} prediction set defined in \eqref{eq:MSRLCP_pred_set}. Before turning to that result, we further examine the role of the perturbed test feature $\tilde X_{n+1}$ in our procedure.

\subsection{Role of $\tilde X_{n+1}$ in \texttt{MS-RLCP}}

The perturbed test feature $\tilde X_{n+1}$ plays a central role in constructing the \texttt{MS-RLCP} prediction set in \eqref{eq:MSRLCP_pred_set}: in \textbf{Step I}, it is used to construct source-specific prediction sets with local coverage guarantees, and in \textbf{Step II}, it is used to identify an appropriate source. We now examine its role in both components and provide justification for its use in both components of the procedure.

\paragraph{Prediction Sets with Local Coverage using $\tilde X_{n+1}$.}
A unifying principle underlying our construction, first illustrated in Section~\ref{sec:warmup} and developed further in Section~\ref{sec:MS-RLCP}, is that each source-specific prediction set aims to satisfy a notion of \emph{local coverage}. We introduce the perturbed test feature $\tilde X_{n+1}$ to define such a notion. A natural question is whether one can avoid this additional randomness and instead seek stronger, test-conditional coverage guarantees. At the extreme, one might hope to construct, for each source $k\in[K]$, a prediction set $\hat C_k$ such that $\mathbb{P}\bigl(Y_{n+1}\in\hat C_k(X_{n+1})\mid X_{n+1}\bigr)\ge 1-\alpha$. This requirement is strictly stronger than those in \eqref{eq:cond_cover} and \eqref{eq:rlcp_target}. If such a construction were possible, choosing \emph{any} source based on the test feature would preserve the desired coverage guarantee.

However, even in the single-source setting, such conditional coverage guarantees are generally impossible without producing uninformative prediction sets, that results in infinite length in regression settings \citep{vovk2005algorithmic}. We sidestep this limitation by adopting a principled notion of local coverage through the perturbed feature $\tilde X_{n+1}$, as in \eqref{eq:rlcp_target}.

While the framework could, in principle, accommodate alternative relaxations of conditional coverage, such as targeting coverage within small balls in the feature space $\mathcal{X}$, we adopt this approach for its tractability. It enables a precise theoretical analysis and the coverage guarantees presented in Section~\ref{sec:msrlcp_theory_main}.

\paragraph{Source Identification through the Perturbed Test Feature $\tilde X_{n+1}$.}
The goal of source identification is to select the source that is most likely to have generated the test feature $X_{n+1}$. Assuming that each source feature distribution $P_{k,X}$ admits a density $f_{k,X}$, a natural oracle strategy would be to select a source in $\arg\max_{k\in[K]} f_{k,X}(X_{n+1})$. In practice, however, accurately estimating these densities under minimal assumptions is challenging, particularly in high-dimensional settings \citep{tsybakov2008nonparametric}. We circumvent this difficulty by introducing the perturbed test feature $\tilde X_{n+1}$. In particular from \eqref{eq:pop_den_Xtilde_k} recall that the marginal density of the perturbed feature under source $k$, evaluated at $\tilde x$, is $\mathbb{E}_{X\sim P_{k,X}}\left[H(X,\tilde x)\right]$. This quantity can be easily estimated at each source using an empirical average, making it useful for source selection. As a result, the perturbed test feature $\tilde X_{n+1}$ provides a principled and computationally tractable way to compare local representation across sources, aligning source selection with the notion of local coverage introduced in \textbf{Step I}.

\section{Theoretical Guarantees}\label{sec:msrlcp_theory_main}
We now establish coverage guarantees for \texttt{MS-RLCP}. We begin by introducing an \emph{envelope distribution} that captures the aggregate representation of the feature space across sources and plays a central role in our theoretical analysis.

\begin{figure}[!t]
    \centering
    \fbox{
    \includegraphics[width=0.5\linewidth]{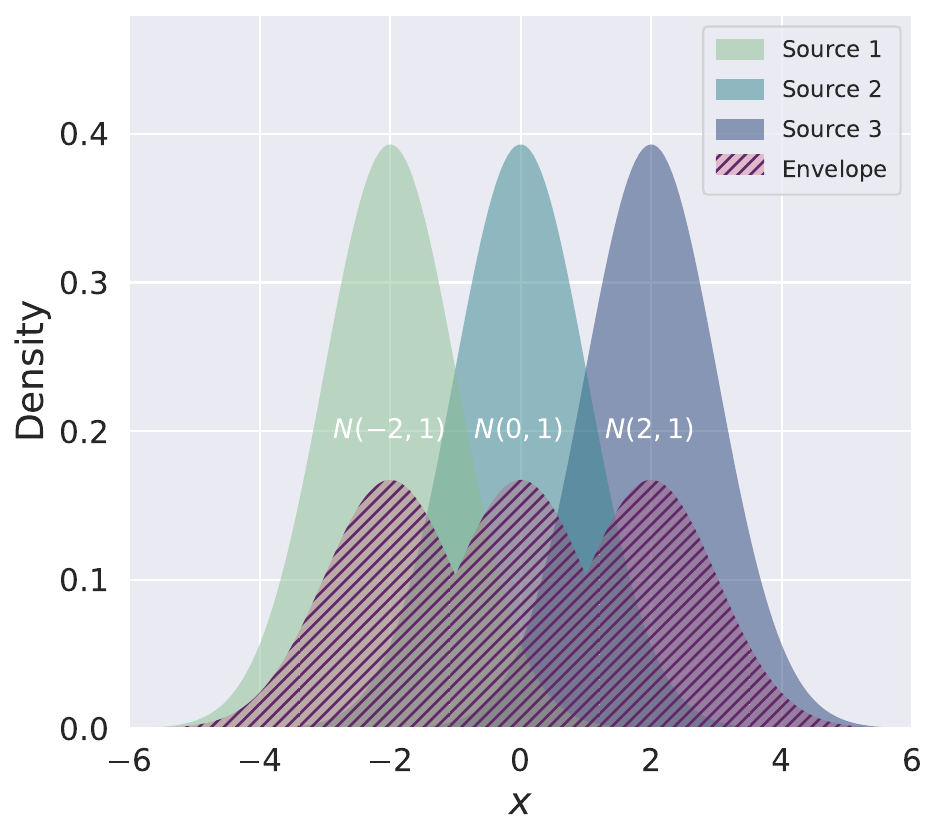}}
    \caption{Illustration of the envelope density for $\cX=\R$ with $K=3$ source distributions $P_{1,X}, P_{2,X}, P_{3,X}$ given by $\rmN(-2,1)$, $\rmN(0,1)$, and $\rmN(2,1)$, respectively, shown as black curves. The corresponding envelope density $\bar f_X$ (shaded) captures the aggregate representation of the feature space across sources.}
    \label{fig:envelope_density}
\end{figure}

Suppose that each source feature distribution $P_{k,X}$ admits a density $f_{k,X}$ with respect to a common dominating measure $\nu$. This imposes no additional assumption, since one can always choose $\nu=\tfrac{1}{K}\sum_{k=1}^{K}P_{k,X}$ and note that for each $k\in [K]$, $P_{k,X}\ll \nu$. With this notation laid out, we define the envelope distribution as follows.

\begin{definition}\label{def:envelope_density}
Define the \emph{envelope density} by $\bar f_X(x)=B^{-1}\max_{k\in[K]}f_{k,X}(x)$, where $B=\int_{\cX}\max_{k\in[K]}f_{k,X}(x)\,\d\nu(x)$ is the normalizing constant. We denote the corresponding \emph{envelope distribution} by $\bar P_X$.
\end{definition}

The name ``envelope'' reflects that the unnormalized density $\max_{k\in[K]}f_{k,X}$ is the smallest function that upper bounds all source densities pointwise. After normalization, it defines a distribution that captures local representation across sources, as illustrated in Figure~\ref{fig:envelope_density}.

By construction, each source feature distribution $P_{k,X}$ is absolutely continuous with respect to $\bar P_X$. Consequently, the density $g_{k,X}=\d P_{k,X}/\d\bar P_X$ is well defined. The following lemma records this fact.

\begin{lemma}\label{lemma:bdd_g_k}
For every $k\in[K]$, the source feature distribution $P_{k,X}$ admits a density $g_{k,X}$ with respect to $\bar P_X$. Moreover, $\max_{k\in[K]}g_{k,X}(x)=B$ for $\bar P_X$-almost every $x$, where $B$ is defined in Definition~\ref{def:envelope_density}.
\end{lemma}
By Lemma~\ref{lemma:bdd_g_k}, we may also write $B=\max_{k\in[K]}\|g_{k,X}\|_\infty$, where $\|g\|_\infty=\operatorname*{ess\,sup}_{x\in\cX}|g(x)|$ denotes the $L_\infty(\bar P_X)$ norm of a measurable function $g:\cX\to\R$. 

For weights $\lambda_k>0$ satisfying $\sum_{k=1}^K\lambda_k=1$, let $Q_{\lambda,X}:=\sum_{k=1}^K\lambda_kP_{k,X}$. The following lemma shows that every such mixture shares the same null sets as the envelope distribution.

\begin{lemma}\label{lemma:envelope_mixture_equivalence}
Let $\lambda_{\min}:=\min_{k\in[K]}\lambda_k>0$. Then,
\[
B\lambda_{\min}\bar P_X
\le Q_{\lambda,X}
\le B\bar P_X.
\]
In particular, $Q_{\lambda,X}$ and $\bar P_X$ are equivalent measures (i.e., absolutely continuous w.r.t. each other).
\end{lemma}

Beyond the covariate shift assumption, we impose no structural assumption on the test distribution $P_{\text{test}}$ except the following:

\begin{assumption}\label{assumption:test_density}
The test feature distribution $P_{\test,X}$ is absolutely continuous with respect to the envelope distribution $\bar P_X$, and we write $g_{\test,X}$ to denote the corresponding density.
\end{assumption}

Assumption~\ref{assumption:test_density} ensures that the test feature distribution assigns no mass to sets that have zero probability under every source. This is necessary for informative distribution-free prediction, since otherwise, the source data can provide no information about the feature--response relationship on those sets.

By Lemma~\ref{lemma:envelope_mixture_equivalence}, any $P_{\text{test}}$ that is absolutely continuous with respect to any mixture of the source distributions (with possibly zero mixture weights) satisfies Assumption~\ref{assumption:test_density}. Existing multi-source conformal methods \citep{yang2026multi,lu2023federated,ying2024informativeness} often assume that the test feature distribution itself is a mixture of the source distributions. Our framework allows a substantially broader class: the test distribution need not be a mixture, but may only be absolutely continuous with respect to one. 

\subsection{Finite-sample bound on coverage}
We now state finite-sample lower bounds on the coverage of
\texttt{MS-RLCP}. For clarity of exposition, we defer the general
result to Theorem~\ref{thm:main_coverage}, stated and proved in
Appendix~\ref{app:general_analysis}, and present a simplified
bound here. Specifically, we assume that the densities
$g_{1,X},\ldots,g_{K,X}$ and $g_{\test,X}$ are globally
$L$-Lipschitz.

Let $n_\eff:=\min_{k\in[K]}n_{k,\train}$ denote the smallest
training split size among the $K$ sources, which we refer to
as the effective training sample size.

\begin{theorem}\label{thm:msrlcp_lipschitz}
Fix $\alpha\in(0,1)$. Suppose $\cX=\R^d$ for some integer
$d\ge1$, Assumption~\ref{assumption:test_density} holds, and
the densities $g_{1,X},\ldots,g_{K,X}$ and $g_{\test,X}$ are
globally $L$-Lipschitz, with $\|g_{\test,X}\|_\infty<\infty$.
Let $H$ be a kernel as in Section~\ref{sec:local_coverage_RLCP},
with $\|H\|_\infty<\infty$ and
$\E_{X\sim\bar P_X,\,\tilde X\mid X\sim H(X,\cdot)}
\|X-\tilde X\|_2<\infty$, and suppose $n_\eff\ge2$.
Then the \texttt{MS-RLCP} prediction set satisfies
\begin{align}\label{eq:lipschitz_kernel_coverage}
\P\bigl(Y_{n+1}\in\hat C_n(X_{n+1})\bigr)
\ge{}&
\left(1-\alpha-\frac1{n_\eff}\right)
-
2L\left(1+\frac{2\|g_{\test,X}\|_\infty}{B}\right)
\E_{\substack{X\sim\bar P_X\\
              \tilde X\mid X\sim H(X,\cdot)}}
\|X-\tilde X\|_2
\nonumber\\
&-
\P_{\substack{X\sim P_{\test,X}\\
              \tilde X\mid X\sim H(X,\cdot)}}
\left(
\max_{k\in[K]}\E_{X'\sim P_{k,X}}
\left[H(X',\tilde X)\right]
\le t_{n_\eff}
\right),
\end{align}
where $t_{n_\eff}:=2\|H\|_\infty
\sqrt{2\log(2Kn_\eff)/n_\eff}$.
\end{theorem}

The proof is given in Appendix~\ref{sec:proof_of_msrlcp_lipschitz}. For further discussion, we simplify this bound for the special case of box kernel.

\begin{corollary}
Consider the setting of Theorem~\ref{thm:msrlcp_lipschitz}, and suppose that $H$ is the box kernel, defined in~\eqref{eq:gauss_box_kernels}. Then,
    \begin{align*}
&\left(1-\alpha-\frac1{n_\eff}\right)
-2Lh\left(1+\frac{2\|g_{\test,X}\|_\infty}{B}\right)-
\P_{\substack{X\sim P_{\test,X}\\
              \tilde X\mid X\sim H(X,\cdot)}}\left(
\max_{k\in[K]}P_{k,X}\bigl(B_h(\tilde X)\bigr)
\le\tau_{n_\eff}
\right),
\end{align*}
where $B_h(x)$ denotes the closed Euclidean ball of radius $h$ centered at $x$, and $\tau_{n_\eff}:= 2\sqrt{\frac{2\log(2Kn_\eff)}{n_\eff}}$.
\end{corollary}
Since for box kernel, $\|X-\tilde X\|_2\le h$ almost surely, the corollary is immediate from Theorem~\ref{thm:msrlcp_lipschitz}. Each of the terms in the bound can be interpreted as follows:

\smallskip
\noindent\textbf{(a)}
The first term is the target coverage level $1-\alpha$,
reduced by a finite-sample correction $1/n_\eff$ that
vanishes as the sample size at every source grows.

\medskip
\noindent\textbf{(b)}
The second term accounts for localization and perturbation
and is of order $O(h)$ for the box kernel. Choosing a smaller
$h$ reduces this term but can lead to larger prediction sets.
A typical approach is to let $h$ vary with $n_\eff$; we return
to this tradeoff later in the section. For a general kernel,
the corresponding quantity is the average perturbation distance
$\E_{X\sim\bar P_X,\,\tilde X\mid X\sim H(X,\cdot)}
\|X-\tilde X\|_2$.

\medskip
\noindent\textbf{(c)}
The final term measures how often the perturbed test feature
falls in a region that is poorly represented across sources.
Under the Lipschitz assumption and the envelope construction,
\[
\max_{k\in[K]}P_{k,X}\bigl(B_h(\tilde x)\bigr)
\ge c_h\bar P_X\bigl(B_h(\tilde x)\bigr),
\qquad
c_h:=(B-2Lh)\vee\frac{B}{K}>0.
\]
Here, $\bar P_X(B_h(\tilde x))$ is the envelope mass within
a ball of radius $h$ around $\tilde x$; a small value indicates
limited representation in that neighborhood.
Consequently, under Assumption~\ref{assumption:test_density}
and the boundedness of $g_{\test,X}$,
\begin{align*}
\P_{\substack{X\sim P_{\test,X}\\
              \tilde X\mid X\sim H(X,\cdot)}}
\left(
\bar P_X\bigl(B_h(\tilde X)\bigr)
\le\frac{\tau_{n_\eff}}{c_h}
\right)\le
\|g_{\test,X}\|_\infty\cdot
\P_{\substack{X\sim\bar P_X\\
              \tilde X\mid X\sim H(X,\cdot)}}
\left(
\bar P_X\bigl(B_h(\tilde X)\bigr)
\le\frac{\tau_{n_\eff}}{c_h}
\right).
\end{align*}
For fixed $h$, $K$, and source and test distributions, we have
$\tau_{n_\eff}/c_h\to0$ as $n_\eff\to\infty$.
Moreover, noting that
$\bar P_X(B_h(\tilde x))/(V_dh^d)$ is the marginal density of
$\tilde{X}$, by the dominated convergence theorem, the final probability term vanishes in large sample setting.

The general result in Theorem~\ref{thm:main_coverage} allows non-Lipschitz densities while retaining the same three components and their corresponding interpretations, with Lipschitz parameters replaced by modulus of continuities defined in \eqref{eq:modulus_of_continuity}.

\begin{remark}
Even the simplified result allows considerable irregularity in the source and test distributions: the Lipschitz assumption concerns only their densities with respect to $\bar P_X$. Thus, their original densities may have discontinuities or other irregularities, provided their variation relative to the envelope is smooth. As an illustration, let $r$ be any symmetric probability density on $[-1,1]$, which need not satisfy standard smoothness assumptions. Consider a two-source prediction problem with
\[
f_{1,X}(x)=\left(1+\frac{x}{2}\right)r(x),
\qquad
f_{2,X}(x)=\left(1-\frac{x}{2}\right)r(x),
\qquad
f_{\test,X}(x)=r(x).
\]
Symmetry ensures that both source densities integrate to one. The corresponding envelope density is $\bar f_X(x)=B^{-1}(1+|x|/2)r(x)$, giving the density ratios
\[
g_{1,X}(x)=\frac{B(1+x/2)}{1+|x|/2},
\qquad
g_{2,X}(x)=\frac{B(1-x/2)}{1+|x|/2},
\qquad
g_{\test,X}(x)=\frac{B}{1+|x|/2}.
\]
These ratios are Lipschitz on $[-1,1]$ regardless of the irregularities in $r$.
\end{remark}

\begin{remark}[Benefit of heterogeneity]
\label{remark:benefit_of_heterogeneity}
    Observe that in Theorem~\ref{thm:msrlcp_lipschitz}, the leading miscoverage term scales as $O(h)$ for moderate sample sizes and holds whenever the test distribution admits a Lipschitz density with respect to the envelope distribution. In contrast, the single-source RLCP result \citep[Theorem~3]{hore2025conformal} yields a similar $O(h)$ rate, but requires Lipschitz smoothness relative to a single source distribution. The class of distributions that are Lipschitz with respect to the envelope is strictly richer, highlighting the advantage of localization in improving robustness to distribution shift by exploiting heterogeneity across sources.
\end{remark}
\subsection{Asymptotic test-conditional coverage}
Since \texttt{MS-RLCP} is built upon the RLCP method, it is worth investigating whether \texttt{MS-RLCP} inherits the strong local coverage properties of RLCP. In particular, here we will study the coverage of \texttt{MS-RLCP} prediction sets, conditional on the true test feature. 

We adopt the multi-source setting with fixed $K$ sources from the previous section. For simplicity of the exposition, we assume balanced source datasets, with each $\cD_k$ containing $2n$ observations, split equally between $\cD_{k,\train}$ and $\cD_{k,\cal}$. Further, we take $\cX=\R^d$, with $d\ge1$, and use the box kernel in \eqref{eq:gauss_box_kernels} with a deterministic bandwidth $h_n>0$ at total sample size $2nK$.

\begin{theorem}\label{thm:conditional_coverage_main}
Fix $\alpha\in (0,1)$ and $x_0\in\R^d$. Under Assumptions~(\ref{A1})--(\ref{A5}), stated formally in Appendix~\ref{app:cond_cov}, we have that
\[
\P_{Y\sim P_{Y\mid X=x_0}}\left(
Y\in\hat C_n(X)
\;\middle|\;
X=x_0
\right)
\longrightarrow1-\alpha\qquad \text{as}~ n\to\infty.
\]
\end{theorem}

While we have deferred the formal description of the assumptions, to put informally, our main assumptions require (1) shrinking bandwidth such that $h_n\to0$ and $nh_n^d\to\infty$ (2) positive envelope density near $x_0$, (3) smoothness of the learned score distributions under total variation measure. Such assumptions are atypical in studying test-conditional coverages, and appear in existing works (for e.g., see~\citep{guan2023localized,hore2025conformal}). The proof of this result is given in Appendix~\ref{app:proof_of_test_cond_cov}, and it follows from a more general finite-sample bound on test conditional coverage, we give in Theorem~\ref{thm:cond_cov_finite_sample}.

\section{Numerical Experiments}

Now, we evaluate \texttt{MS-RLCP} on both synthetic and real datasets. In all simulations, we use a Gaussian kernel (see~\eqref{eq:gauss_box_kernels}) and employ the smoothed RLCP construction for source-specific prediction sets to reduce over-coverage (see Appendix~B of \cite{hore2025conformal} for more details).

\subsection{Simulations}\label{sec:sims}
We first evaluate \texttt{MS-RLCP} on a synthetic setting designed to induce systematic heterogeneity. Let $\cX=\R^5$ and consider $K=10$ source distributions with taking $P_{k,X}=\mathcal{N}(\nu_k,I_5)$. To construct $\{\nu_k\}$, enumerate all $10$ pairs $(i_1,i_2)$ with $1\le i_1<i_2\le 5$, and for the $k$th pair $(i_1,i_2)$ we define $\nu_k = (\nu_{k,1},\ldots,\nu_{k,5})$ by
\[
(\nu_k)_j=
\begin{cases}
    -\nu, & j\in \{i_1,i_2\},\\
    \ \nu, & \text{otherwise}.
\end{cases}
\]
This construction induces structured heterogeneity, with each source emphasizing a distinct subpopulation of the feature space. The conditional model is shared across sources and is given by \[Y=\sin(X_1+X_2)+\cos(X_3)+X_4+2X_5+\epsilon\] with  $\epsilon\sim\mathcal{N}(0,1)$. 

\paragraph{Experiment 1: beyond mixture settings.}
In Section~\ref{sec:msrlcp_theory_main} we establish robust coverage of \texttt{MS-RLCP} beyond the standard mixture setting (i.e., where the test population is a mixture of the $10$ source distributions). We define $P_{\rm test}$ with the same conditional $P_{Y\mid X}$ and $P_{{\rm test},X}=\frac{1}{10}\sum_{k=1}^{10}\mathcal{N}(\nu_k^{(\mu)},I_5)$,
where $\nu_k^{(\mu)}$ is defined analogously to $\nu_k$ with $\nu$ replaced by $\mu$. The case $\nu=\mu$ recovers the exact mixture setting with uniform weights, while varying $\mu\in[-2,2]$ induces controlled departures from it. For each source, we generate $5,000$ training observations and evaluate all methods on $1,000$ test samples. The experiment is repeated $100$ times, and we report the empirical coverage and average prediction set length across repetitions. As shown in Figure~\ref{fig:expt1} left panel, \texttt{MS-RLCP} maintains coverage close to the target level across all $\mu$.

\begin{figure}[!t]
    \centering
    \includegraphics[width=\linewidth]{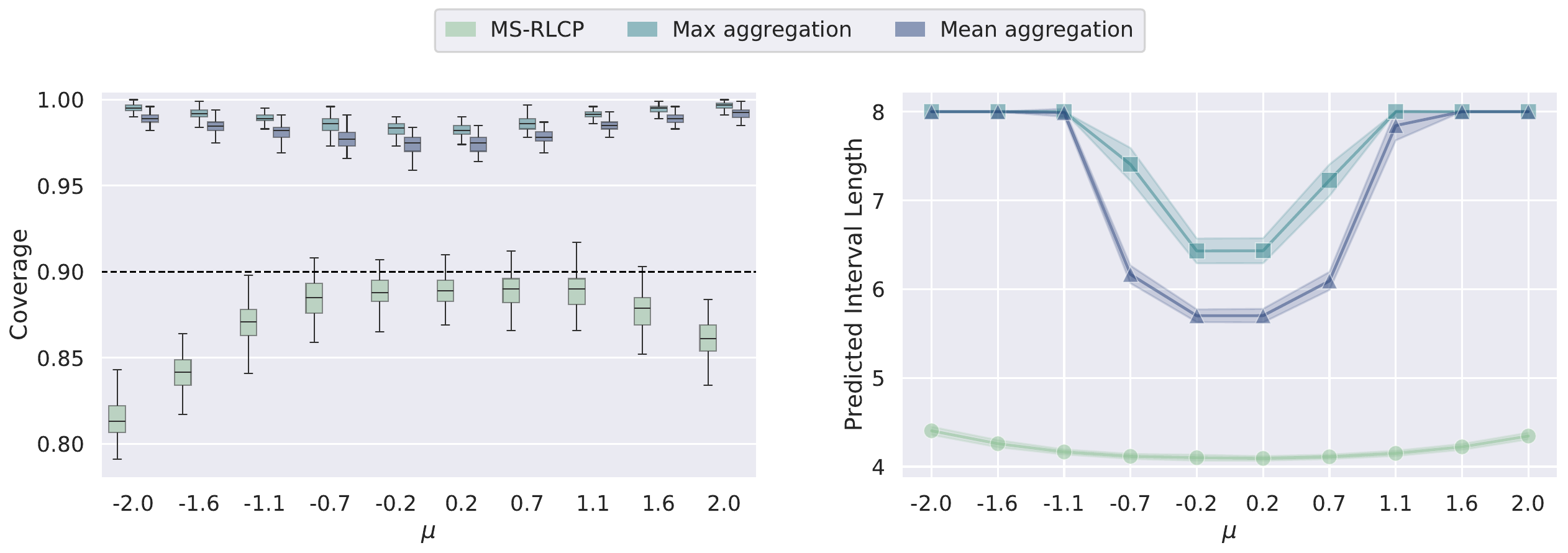}
    \caption{\texttt{MS-RLCP} achieves coverage close to the target level $0.9$ while producing substantially shorter prediction intervals than competing methods. Left: empirical coverage in Experiment~1. Right: average prediction set length in Experiment~1.}
    \label{fig:expt1}
\end{figure}

Since \texttt{MS-RLCP} can be cast as a data-dependent aggregation of source-specific confidence sets, it is natural to compare \texttt{MS-RLCP} with other aggregation-based baselines, i.e., applying RLCP within each source to obtain $p$-values $(p_k)_{k\in[K]}$, followed by a \emph{max}- and a \emph{mean}-aggregation \citep{yang2026multi,ying2024informativeness}. These baselines while natural are highly conservative (i.e., empirical coverage is close to $1$), whereas \texttt{MS-RLCP} attains near-nominal coverage ($0.9$) with substantially shorter prediction sets (Figure~\ref{fig:expt1} right panel)

\paragraph{Experiment 2: benefit of heterogeneity.}
As suggested in Remark~\ref{remark:benefit_of_heterogeneity}, \texttt{MS-RLCP} leverages heterogeneity via localization. To assess this, we generate the test distribution as in Experiment~1 and vary $\nu\in\{0,0.5,1.0\}$, where note that larger $\nu$ corresponds to greater heterogeneity across sources. Figure~\ref{fig:expt2} left panel shows that for small $\nu$, coverage becomes conservative as $\mu$ increases, indicating sensitivity to distribution shift. In contrast, for larger $\nu$, coverage remains close to the target level even for extreme $\mu$, confirming that heterogeneity improves robustness. Moreover, for larger $\nu$, \texttt{MS-RLCP} yields noticeably shorter prediction sets, reflecting improved learning of the feature--response relationship across the heterogeneous sources (Figure~\ref{fig:expt2} right panel).

\begin{figure}[!t]
    \centering
    \includegraphics[width=0.85\textwidth]{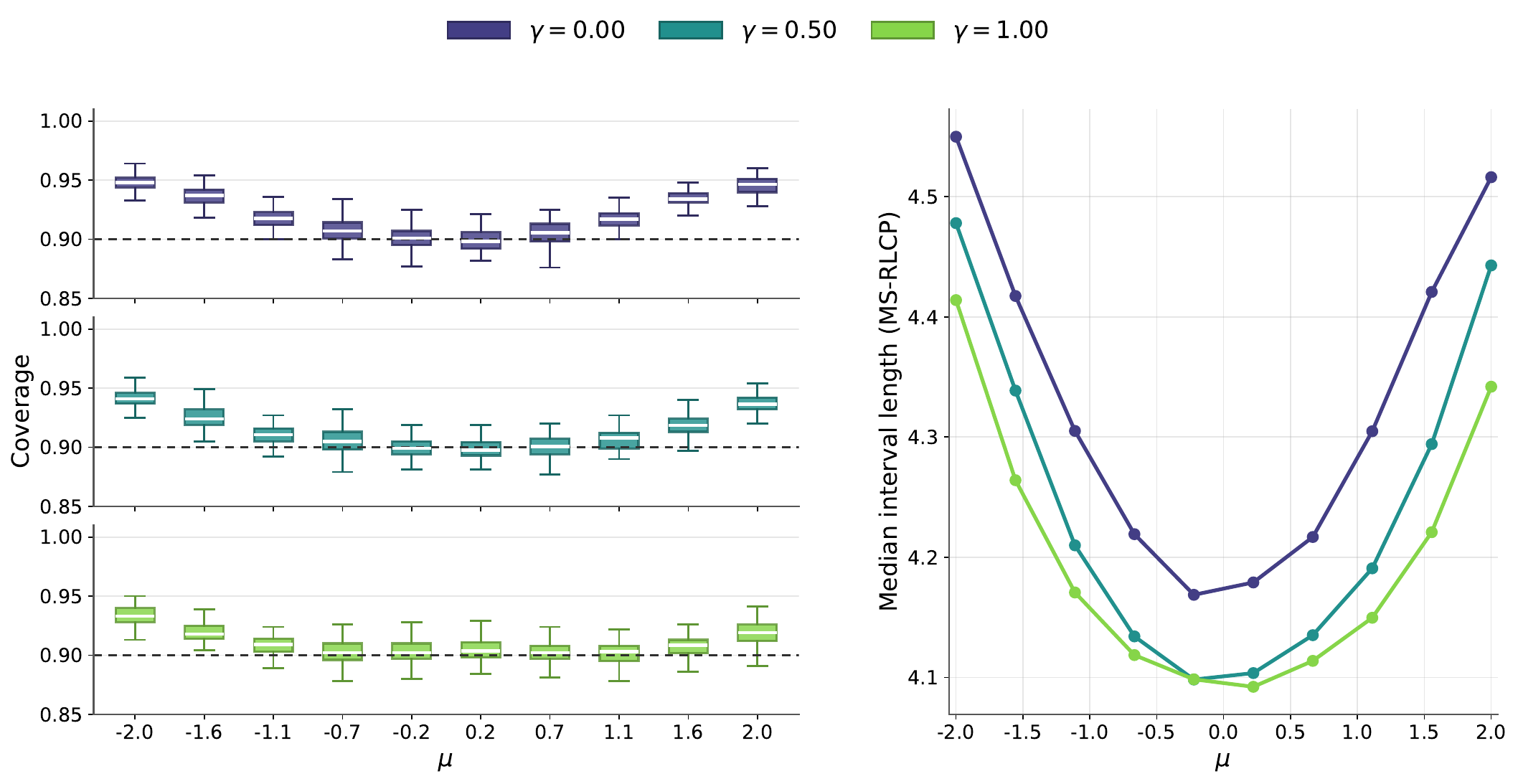}
    \caption{ Heterogenrity results in more robust coverage and shorter set-lengths for \texttt{MS-RLCP}. Left: empirical coverage in Experiment~2. Right: average prediction set length in Experiment~2.}
    \label{fig:expt2}
\end{figure}

\begin{figure*}[t]
    \centering
    \captionsetup[subfigure]{font=small,skip=2pt}

    \begin{subfigure}[t]{0.40\linewidth}
        \centering
        \includegraphics[width=0.9\linewidth]{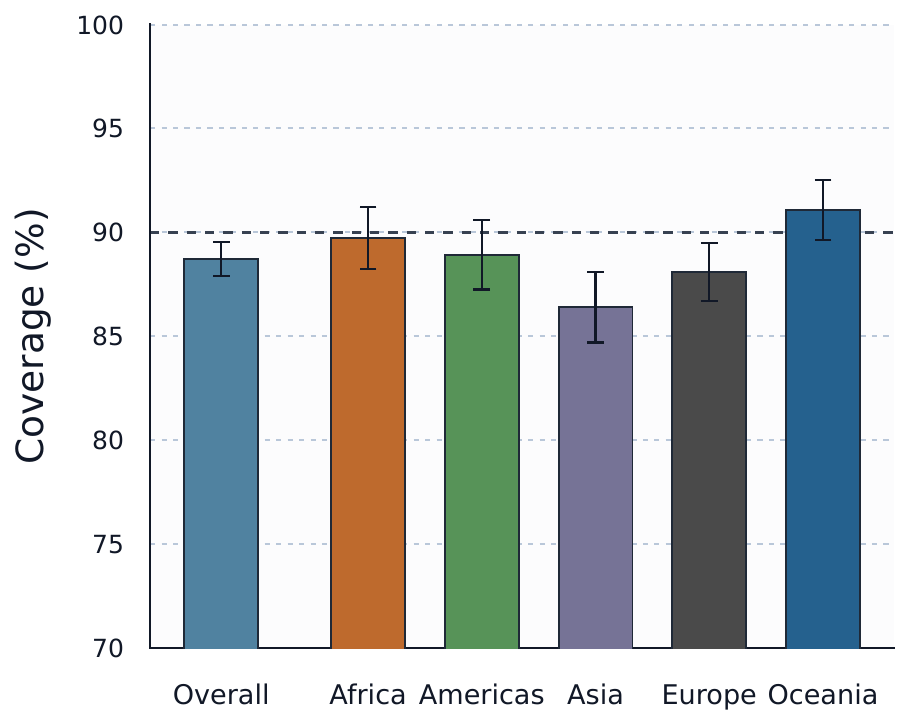}
        \caption{FMoW: Coverage}
        \label{fig:fmow_barplots_coverage}
    \end{subfigure}%
    \begin{subfigure}[t]{0.60\linewidth}
        \centering
        \includegraphics[width=0.9\linewidth]{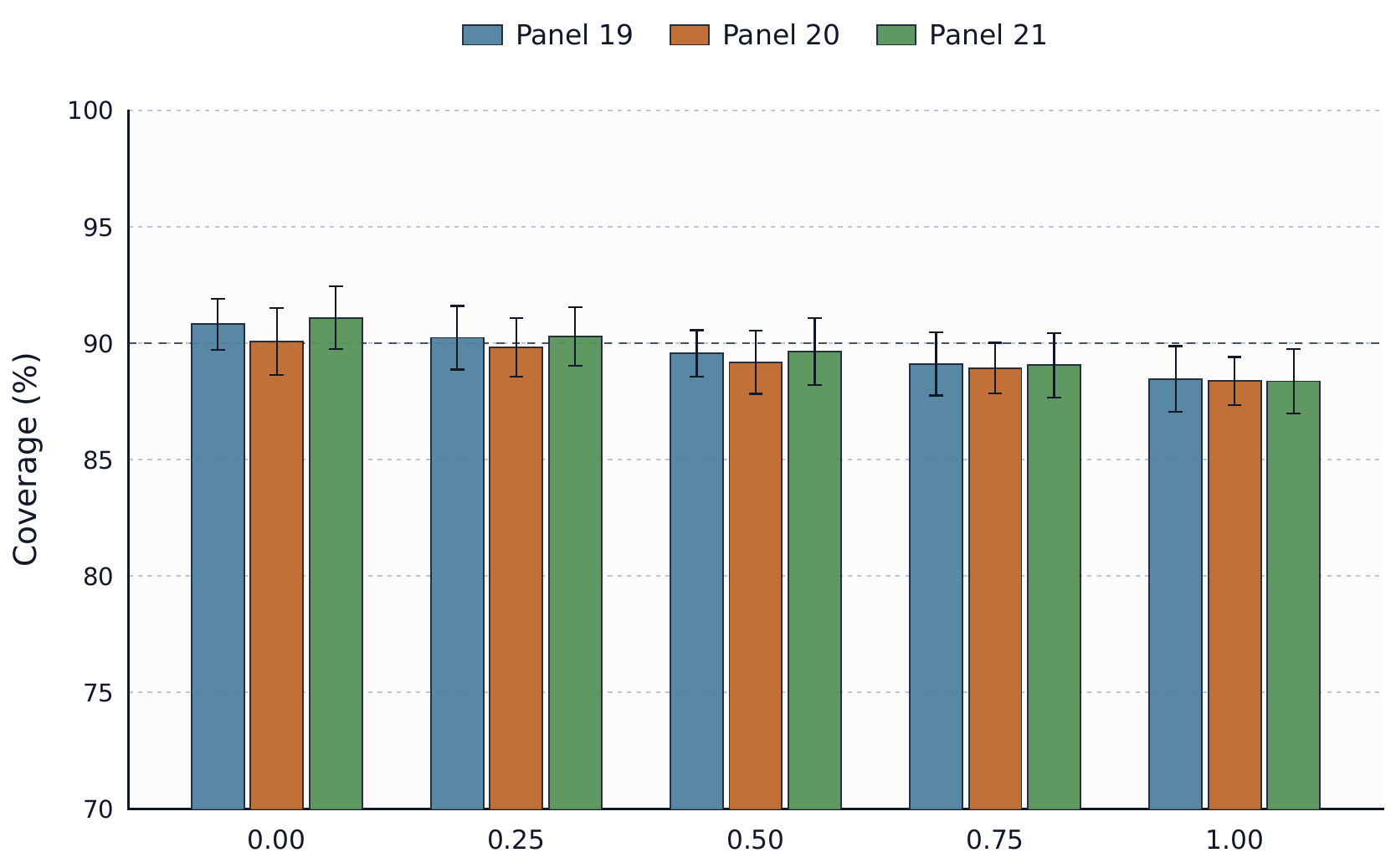}
        \caption{MEPS: Coverage}
        \label{fig:meps_eta_coverage}
    \end{subfigure}

    \begin{subfigure}[t]{0.40\linewidth}
        \centering
        \includegraphics[width=0.9\linewidth]{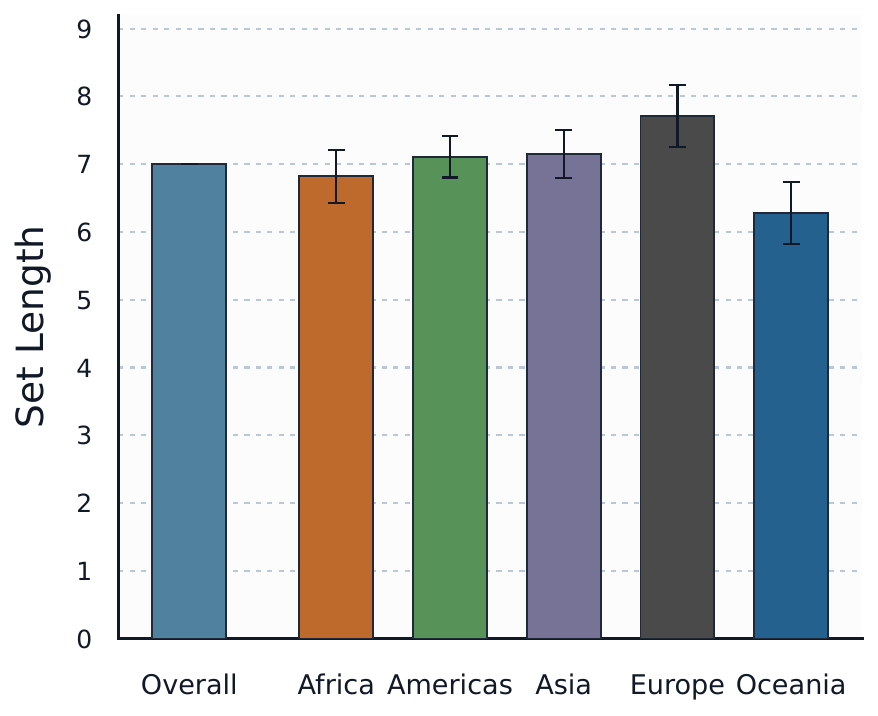}
        \caption{FMoW: Set Length}
        \label{fig:fmow_barplots_width}
    \end{subfigure}%
    \begin{subfigure}[t]{0.60\linewidth}
        \centering
        \includegraphics[width=0.9\linewidth]{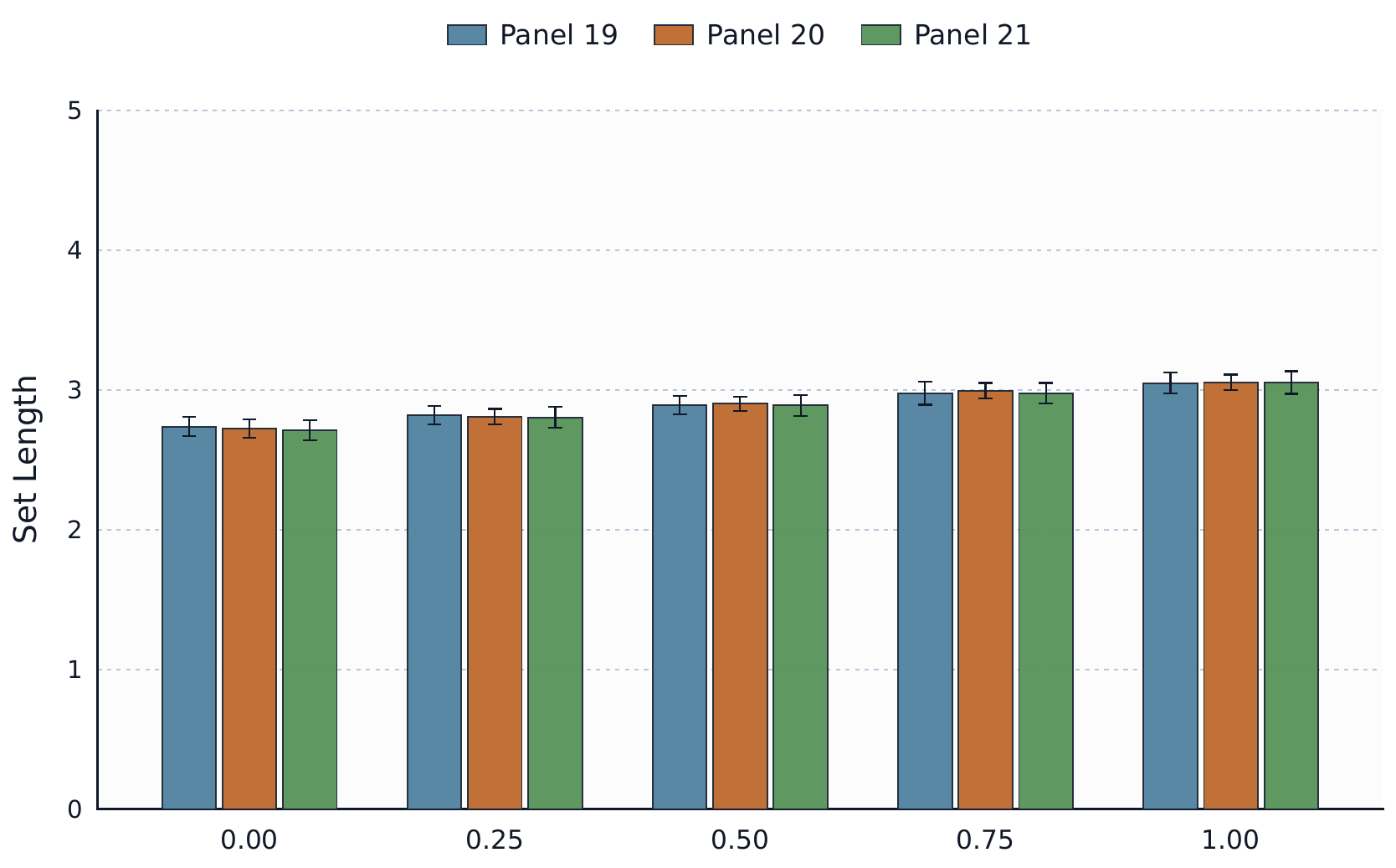}
        \caption{MEPS: Set Length}
        \label{fig:meps_eta_width}
    \end{subfigure}

    \caption{\small Figure \ref{fig:fmow_barplots_coverage} and \ref{fig:fmow_barplots_width} plot coverage and set length for FMOW experiment. Figure \ref{fig:meps_eta_coverage} and \ref{fig:meps_eta_width} plot the effect of latent test-mixture shift in the MEPS experiment.
    }
    \label{fig:four_panel}
    
\end{figure*}

\subsection{Real Data Experiments}\label{sec:real_experiment}

\subsubsection{FMoW Classification under Geographic Distribution Shift}\label{sec:fmow_experiments}

Machine learning models are widely used for satellite imagery tasks such as land-use mapping, resource allocation, and risk assessment \citep{cong2022satmae, jean2016combining, wang2018deep, russwurm2020self}. A key challenge for reliable inference in such settings is subpopulation shift due to geographic and imaging variability. We evaluate \texttt{MS-RLCP} to tackle such distribution shifts and perform reliable predictive inference from satellite imagery.

We use the Functional Map of the World (2016 slice) dataset \citep{christie2018functional}, covering 249 countries/regions and 62 classes. We define Africa, the Americas, Asia, Europe, and Oceania as source domains. From each, 40\% of samples are used to train a shared \texttt{DenseNet} feature extractor \citep{huang2017densely} (ImageNet-initialized). Remaining data are split into train, calibration, and test sets. Source-specific heads are trained on PCA-reduced (16 components) features. We compute conformal scores using RAPS \citep{ding2023class, angelopoulosuncertainty} and construct RLCP prediction sets from the calibration split with target coverage $1-\alpha = 0.9$.




For the test set, we pool test partitions across all sources, which simulates a mixture of the source populations for the test population. Over 50 repetitions, the method achieves $88.59 \pm 0.82\%$ overall coverage with median set width $7.12 \pm 0.33$ labels. We also compute coverage of \texttt{MS-RLCP} within each region. From Figure \ref{fig:fmow_barplots_coverage}, the coverage in Africa, the Americas, and Oceania are near or above the $0.9$ target, while Asia and Europe are more challenging. Correspondingly, from Figure \ref{fig:fmow_barplots_width}, set sizes increase in these harder regions, indicating that the method adapts uncertainty locally rather than enforcing a single global set size. Additional experimental details are provided in Appendix~\ref{app:fmow_details}.

\subsubsection{MEPS Healthcare Utilization}\label{subsec:meps}
We now evaluate \texttt{MS-RLCP} on the MEPS healthcare-utilization benchmark \citep{romano2020malice} (panels 19–21), using sources defined by the sensitive attribute race (White vs non-White) as in prior works \cite{romano2020malice} and \cite{yang2026multi}. The goal is to predict medical care utilization. We retain continuous features and apply log transforms to both features and response.
Within each source, we split data into train, calibration, and test sets and fit a heteroskedastic Gaussian gradient-boosted model (see \eqref{eq:het_gauss}). Calibration scores are computed as absolute residuals normalized by estimated conditional standard deviation. We target coverage $1-\alpha=0.9$ and pool test data across sources, reporting overall and source-specific coverage as well as interval width (on the log-transformed scale).
\begin{table}[!t]
\centering
\small
\setlength{\tabcolsep}{4pt}
\caption{\small MEPS results averaged over $50$ repetitions. }
\label{tab:meps-results}
\begin{tabular*}{\linewidth}{@{\extracolsep{\fill}}lcccccc}
\toprule
\multirow{2}{*}[-2pt]{\textbf{Panel}}
& \multicolumn{3}{c}{\textbf{Coverage (\%)}}
& \multicolumn{3}{c}{\textbf{Set Length}} \\
\cmidrule(lr){2-4} \cmidrule(lr){5-7}
& \textbf{Overall}
& \textbf{White}
& \textbf{Non-White}
& \textbf{Overall}
& \textbf{White}
& \textbf{Non-White} \\
\midrule
19 & 89.8 & 88.3 & 90.7 & 2.86 & 3.04 & 2.72 \\
20 & 89.4 & 88.2 & 90.1 & 2.89 & 3.05 & 2.71 \\
21 & 89.7 & 88.1 & 90.8 & 2.88 & 3.05 & 2.71 \\
\bottomrule
\end{tabular*}%
\end{table}

Table~\ref{tab:meps-results} confirms that coverage remains close to the target level $0.9$ overall and within each source. We further evaluate robustness to latent mixture shift by varying the White fraction $\eta \in \{0,0.25,0.5,0.75,1\}$ in synthetic test sets (with fixed total size per panel). Figure~\ref{fig:meps_eta_coverage} and Figure \ref{fig:meps_eta_width} show that coverage and median interval length remain stable across all mixture proportions, with slightly increased difficulty as the test set becomes more White-dominated, consistent with the smaller sample size in that source. Additional details are presented in Appendix \ref{app:meps-details}.

\section{Discussion}
In this work, we propose \texttt{MS-RLCP}, which combines the local
coverage properties of single-source RLCP with data-adaptive source
selection. We introduce an envelope distribution that captures the
collective representation of the feature space across sources and
provides a common reference for theoretical analysis. Under a shared
conditional response distribution $P_{Y\mid X}$ and absolute continuity
of the test feature distribution with respect to the envelope, we derive
interpretable coverage lower bound.

A natural direction for future work is to extend source selection to
local weighting of multiple sources. Such an extension could draw on
more calibration observations near each test point, potentially improving
the efficiency of the prediction sets while retaining coverage guarantees.
\section*{Acknowledgments}
During the preparation of this manuscript, the authors used AI tools for coding assistance, proofreading, and improving the manuscript's
presentation. The authors have carefully reviewed the manuscript and
take full responsibility for its content and conclusions.

\bibliographystyle{abbrvnat}
\bibliography{reference}

\newpage
\appendix

\part*{Appendix}
\tableofcontents

\newpage

\section{Proof of results from Section~\ref{sec:msrlcp_theory_main}}
\subsection{Proof of Lemma~\ref{lemma:bdd_g_k}}
Fix any $k\in [K]$. For any measurable set $A\subseteq \mathcal{X}$, we have that
\[
\P_{k,X}(A)\le \int_A \max_{k\in [K]}f_k(x)\, \d\nu(x)=B\int_A \bar{f}(x)\, \d\nu(x)=B\bar{P}_X(A).
\]
Consequently, $\P_{k,X}\ll \bar{P}_X$ for each $k\in [K]$, and hence $g_k=\frac{\mathsf{d}P_{k,X}}{\mathsf{d}\bar{P}_X}$ exists. Moreover, by construction, for any $x\in \cX$ and for any $k\in [K]$,
\[
g_{k}(x)\,\le\, \max_{k\in [K]} g_k(x)=\max_{k\in [K]} \frac{\frac{\d\P_{k,X}}{\d\nu}}{\frac{\d\bar{P}_X}{\d\nu}}(x)=\max_{k\in [K]}\frac{f_k(x)}{\bar{f}(x)}=B.
\]
This proves the second part. $\hfill\square$

\subsection{Proof of Lemma~\ref{lemma:envelope_mixture_equivalence}}

Since $\lambda_k\ge\lambda_{\min}$ and $\sum_{k=1}^K\lambda_k=1$, for every $x\in \mathcal{X}$,
\[
\lambda_{\min}\max_{k\in[K]}f_{k,X}(x)
\le
\sum_{k=1}^K\lambda_kf_{k,X}(x)
\le
\max_{k\in[K]}f_{k,X}(x).
\]
Using $\max_{k\in[K]}f_{k,X}=B\bar f_X$ and integrating over any measurable set $A\subseteq \mathcal{X}$ gives the stated measure inequalities. Since $B\lambda_{\min}>0$, the two measures have the same null sets and are therefore equivalent measures. $\hfill\square$

\subsection{Proof of Theorem~\ref{thm:msrlcp_lipschitz}}
\label{sec:proof_of_msrlcp_lipschitz}

We start with the coverage decomposition in~\eqref{eq:lb_main_1}
from the proof of Theorem~\ref{thm:main_coverage}, which gives
\[
\P\bigl(Y_{n+1}\in\hat C_n(X_{n+1})\bigr)
\ge
1-\alpha-\Delta_{\cG,2}-\Delta_{\cG,1}.
\]
$\Delta_{\cG,2}$ can be controlled by
Lemma~\ref{lemma:Delta_2_bdd}. In particular,
\[
\Delta_{\cG,2}
\le
\inf_{0<\delta<1}
\left\{
\delta+
\P_{\tilde X_{n+1}\sim P_{\test,\tilde X}}
\left(
\mu_{k_\oracle(\tilde X_{n+1})}(\tilde X_{n+1})
\le t_\delta
\right)
\right\}.
\]
With the choice of $\delta=\frac{1}{n_\eff}$, $t_\delta$ reduces to $t_{n_\eff}$ where the latter is as defined in the theorem statement. Further, noting that 
\[
\mu_{k_\oracle(\tilde X_{n+1})}(\tilde X_{n+1})=\max_{k\in[K]}\E_{X'\sim P_{k,X}}
\left[H(X',\tilde X_{n+1})\right],
\]
we obtain
\[
\Delta_{\cG,2}\le \frac{1}{n_\eff}+\P_{\substack{X\sim P_{\test,X}\\
              \tilde X\mid X\sim H(X,\cdot)}}
\left(
\max_{k\in[K]}\E_{X'\sim P_{k,X}}
\left[H(X',\tilde X)\right]
\le t_{n_\eff}
\right).
\]
On the other hand, Lemma~\ref{lemma:Delta_G1_bdd_lipschitz}, yields
\[
\Delta_{\cG,1}
\le
2L\left(1+\frac{2\|g_{\test,X}\|_\infty}{B}\right)
\E_{\substack{X\sim\bar P_X\\
              \tilde X\mid X\sim H(X,\cdot)}}
\|X-\tilde X\|_2.
\]
Substituting these two bounds into~\eqref{eq:lb_main_1}
proves~\eqref{eq:lipschitz_kernel_coverage}.
\hfill$\square$

\subsection{Proof of Theorem~\ref{thm:conditional_coverage_main}}\label{app:proof_of_test_cond_cov}
By Assumption~(\ref{A4}), choose $c>0$ and $r>0$ such that $\bar f_X(x)\ge c$ for almost every $x\in B_r(x_0)$. By Assumption~(\ref{A3}), $2h_n<r$ for all sufficiently large $n$. For every $\tilde x\in B_{h_n}(x_0)$, we then have
\begin{align*}
\max_{k\in[K]}P_{k,X}(B_{h_n}(\tilde x))
&=
\max_{k\in[K]}
\int_{B_{h_n}(\tilde x)}f_{k,X}(x)\,\d x\\
&\ge
\frac1K
\int_{B_{h_n}(\tilde x)}
\sum_{k=1}^K f_{k,X}(x)\,\d x\ge
\frac{B}{K}
\int_{B_{h_n}(\tilde x)}\bar f_X(x)\,\d x\ge
\frac{BcV_d}{K}h_n^d.
\end{align*}
The first expectation in \eqref{eq:cond_cov_finite_sample_bound} is therefore bounded by
\[
2K\exp\left(-\frac{BcV_d}{36K}nh_n^d\right)
+
\frac{2K}{(n+1)BcV_dh_n^d},
\]
which tends to zero because $nh_n^d\to\infty$.

The Kolmogorov disyance in \eqref{eq:cont_tv} is bounded by one. Its convergence in probability under Assumption~(\ref{A5}) therefore implies convergence of its expectation to zero. Applying Theorem~\ref{thm:cond_cov_finite_sample} completes the proof.$\hfill\square$

\section{General finite-sample coverage guarantee for \texttt{MS-RLCP}}\label{app:general_analysis}

We establish the main coverage theorem for \texttt{MS-RLCP}, which underlies the simplified guarantees in the main text. The bound separates the coverage error into two parts: one arising from data-dependent source selection and the other from differences between the source and test distributions within local neighborhoods.

We retain the setup and notation from Section~\ref{sec:msrlcp_theory_main}. In particular, $\bar P_X$ denotes the envelope distribution, $B$ its normalizing constant, and $g_{k,X}$ and $g_{\test,X}$ the source and test feature densities with respect to $\bar P_X$. The conditional response distribution $P_{Y\mid X}$ is shared across the sources and the test environment.

For each source, the perturbation kernel $H$ induces a marginal distribution $P_{k,\tilde X}$ and a conditional distribution $P_{k,X\mid\tilde X}$. We use analogous notation for the test and envelope distributions. Their perturbed feature densities, with respect to the base measure $\mu$ of the kernel, are
\[
\mu_k(\tilde x):=\E_{P_{k,X}}[H(X,\tilde x)],
\qquad
\mu_{\test}(\tilde x):=\E_{P_{\test,X}}[H(X,\tilde x)],
\qquad
\bar\mu(\tilde x):=\E_{\bar P_X}[H(X,\tilde x)].
\]

To describe local variation in the density ratios, for a measurable set $A\subseteq\cX$, define
\begin{equation}\label{eq:modulus_of_continuity}
\omega_{k,A}
:=
\sup_{\substack{x,y\in A\\x\ne y}}
\frac{|g_{k,X}(x)-g_{k,X}(y)|}{\|x-y\|_2},
\qquad
\omega_{\test,A}
:=
\sup_{\substack{x,y\in A\\x\ne y}}
\frac{|g_{\test,X}(x)-g_{\test,X}(y)|}{\|x-y\|_2},
\end{equation}
and write $\omega_A:=\max_{k\in[K]}\omega_{k,A}$. These are local Lipschitz seminorms, with the supremum over an empty set taken to be zero. For $\epsilon>0$, also define
\begin{equation}\label{eq:def_Delta_x_A_vep}
\Delta(\tilde x;A,\epsilon)
:=
\bar P_{X\mid\tilde X=\tilde x}(A^c)
+
\bar P_{X\mid\tilde X=\tilde x}
\bigl(\|X-\tilde x\|_2>\epsilon\bigr).
\end{equation}
This quantity measures the conditional envelope mass outside $A$ or outside an $\epsilon$-neighborhood of $\tilde x$. Together, the seminorms and $\Delta$ allow us to control local variation on $A$ while accounting for the remaining probability mass.

We use $\|g_{\test,X}\|_\infty$ and $\|H\|_\infty$ for the essential suprema with respect to $\bar P_X$ and $\bar P_X\otimes\mu$, respectively. We fix versions of the densities satisfying their essential bounds and interpret fractions with zero denominators as $+\infty$. Infima over $A$ below are over measurable sets; if a pointwise infimum is not measurable, its expectation is interpreted as an outer expectation.

\begin{theorem}[General coverage guarantee]\label{thm:main_coverage}
Fix $\alpha\in(0,1)$. Suppose Assumption~\ref{assumption:test_density} holds, $\|g_{\test,X}\|_\infty<\infty$, and $\|H\|_\infty<\infty$. Further, define
\[
n_\eff:=\min_{k\in[K]}n_{k,\train},
\qquad
t_\delta
:=
2\|H\|_\infty
\sqrt{\frac{2\log(2K/\delta)}{n_\eff}},
\quad 0<\delta<1.
\]
Then the \texttt{MS-RLCP} prediction set satisfies
\begin{align}\label{eq:main_bdd_delta_1}
\P\bigl(Y_{n+1}&\in\hat C_n(X_{n+1})\bigr)
\ge{}1-\alpha -
\inf_{0<\delta<1}
\left\{
\delta+
\P_{\tilde X_{n+1}\sim P_{\test,\tilde X}}
\left(
\mu_{k_\oracle(\tilde X_{n+1})}(\tilde X_{n+1})
\le t_\delta
\right)
\right\}
\nonumber\\
&-
\E_{\tilde X_{n+1}\sim\bar P_{\tilde X}}
\left[
\inf_{\substack{A\subseteq\cX\\\epsilon>0}}
\frac{
2\left(
\frac{\|g_{\test,X}\|_\infty}{B}\omega_A
+\omega_{\test,A}
\right)\epsilon
+
2\|g_{\test,X}\|_\infty
\Delta(\tilde X_{n+1};A,\epsilon)
}{
\left(
1-\frac{2\omega_A\epsilon}{B}
-2\Delta(\tilde X_{n+1};A,\epsilon)
\right)\vee \frac{1}{K}
}
\right].
\end{align}
Here, $k_\oracle$ is as defined in \eqref{eq:k_oracle}, and the probability on the left is taken over the randomness of all source data $\cD$, the test sample $(X_{n+1},Y_{n+1})$, and the perturbed test $\tilde{X}_{n+1}$.
\end{theorem}

\begin{proof}
We start by recalling that by \eqref{eq:MSRLCP_pred_set}, the \texttt{MS-RLCP} prediction set is obtained by reporting the RLCP prediction set for the source $\hat{k}$. We write $\cD_\train=\bigcup_{k\in[K]}\cD_{k,\train}$ and, for brevity, $\hat k=\hat k(\tilde X_{n+1})$. Then, the \texttt{MS-RLCP} prediction set is given by
\[
\hat C_n(X_{n+1})
=
\hat C_{\hat k}(X_{n+1},\tilde X_{n+1})
=
\hat C_{n_{\hat k}}^{\mathrm{RLCP}}
\left(
X_{n+1},\tilde X_{n+1};
\cD_{\hat k,\train},\cD_{\hat k,\cal}
\right).
\]
Conditional on $\tilde X_{n+1}$ and $\cD_\train$, the selected source $\hat k$ is fixed. Since source selection does not use the calibration data, under the above conditioning, the selected source's calibration observations are still i.i.d.\ from $P_{\hat{k}}$.

We first define the set
\begin{equation}\label{eq:choice_G_set}
\cG
:=
\left\{
\tilde x\in\cX:
\mu_{\hat k(\tilde x)}(\tilde x)
>
\frac12\mu_{k_\oracle(\tilde x)}(\tilde x)
\right\}.
\end{equation}
This set depends on the training data $\cD_\train$ and contains the perturbations for which the `selected source is roughly as good as the oracle choice'.

Fix $\tilde x\in\cG$ and $\cD_\train$. Then $\mu_{\hat k}(\tilde x)>0$, so the conditional distribution $P_{\hat k,X\mid\tilde X=\tilde x}$ is well defined. Now, draw
$X'\sim P_{\hat k,X\mid\tilde X=\tilde x}$ and then
$Y'\mid X'\sim P_{Y\mid X}(\cdot\mid X')$,
independently of the calibration data. Applying Lemma~\ref{lemma:rlcp_hore_barber_validity} to the selected source then gives
\[
\P\left(
Y'\in\hat C_{\hat k}(X',\tilde x)
\;\middle|\;
\tilde X_{n+1}=\tilde x,\cD_\train
\right)
\ge1-\alpha,
\]
where the probability is over $(X',Y')$ and $\cD_{\hat k,\cal}$. Note that conditioning on the other sources' training data does not affect this guarantee, since those data are independent of the selected source's calibration data and the auxiliary sample $(X',Y')$.

Our goal, however, is to control coverage for the actual test sample, whose corresponding conditional feature distribution is $P_{\test,X\mid\tilde X=\tilde x}$. Under covariate shift, both the auxiliary sample and the test sample share the same conditional response law $P_{Y\mid X}$. 

Thus, the difference between their coverage probabilities is bounded by the total variation distance between their conditional feature distributions. Consequently,
\begin{align*}
\P\left(
Y_{n+1}\notin\hat C_{\hat k}(X_{n+1},\tilde x)
\;\middle|\;
\tilde X_{n+1}=\tilde x,\cD_\train
\right)\le
\alpha+
\TV\left(
P_{\test,X\mid\tilde X=\tilde x},
P_{\hat k,X\mid\tilde X=\tilde x}
\right).
\end{align*}

For perturbations $\tilde{x}$ outside $\cG$, we use the trivial upper bound of one on miscoverage. Applying the tower property and splitting according to whether $\tilde X_{n+1}\in\cG$, we obtain
\begin{align*}
\P\bigl(Y_{n+1}\notin\hat C_n(X_{n+1})\bigr)
&\le
\alpha\,\P(\tilde X_{n+1}\in\cG)
+\Delta_{\cG,1}+\Delta_{\cG,2}\\
&\le
\alpha+\Delta_{\cG,1}+\Delta_{\cG,2},
\end{align*}
where we define
\begin{align*}
\Delta_{\cG,2}
&:=
\P_{P_{\test,\tilde X}\times\cD_\train}
\left(\tilde X_{n+1}\notin\cG\right),\\
\Delta_{\cG,1}
&:=
\E_{P_{\test,\tilde X}\times\cD_\train}
\left[
\TV\left(
P_{\test,X\mid\tilde X=\tilde X_{n+1}},
P_{\hat k,X\mid\tilde X=\tilde X_{n+1}}
\right)
\one_{\{\tilde X_{n+1}\in\cG\}}
\right].
\end{align*}
The inner term defining $\Delta_{\cG,1}$ is taken to be zero outside $\cG$. Equivalently,
\begin{equation}\label{eq:lb_main_1}
\P\bigl(Y_{n+1}\in\hat C_n(X_{n+1})\bigr)
\ge1-\alpha-\Delta_{\cG,2}-\Delta_{\cG,1}.
\end{equation}

It remains to bound these two error terms. Lemma~\ref{lemma:Delta_2_bdd} controls $\Delta_{\cG,2}$, the probability of selecting a source with insufficient local mass. Lemma~\ref{lemma:Delta_G1_bdd} controls $\Delta_{\cG,1}$, the expected discrepancy between the conditional feature distributions of the test population and selected source  
population on $\cG$. Substituting their bounds into \eqref{eq:lb_main_1} completes the proof.
\end{proof}

\subsection{Supporting lemmas and their proofs}

We prove the two lemmas used to control $\Delta_{\cG,2}$ and $\Delta_{\cG,1}$ in the proof of Theorem~\ref{thm:main_coverage}.

\begin{lemma}\label{lemma:Delta_2_bdd}
Under the conditions of Theorem~\ref{thm:main_coverage},
\[
\Delta_{\cG,2}
\le
\inf_{0<\delta<1}
\left\{
\delta+
\P_{\tilde X_{n+1}\sim P_{\test,\tilde X}}
\left(
\mu_{k_\oracle(\tilde X_{n+1})}(\tilde X_{n+1})
\le t_\delta
\right)
\right\},
\]
where $t_\delta$ is as defined in Theorem~\ref{thm:main_coverage}.
\end{lemma}

\begin{proof}
For brevity, we write $\mu_k=\mu_k(\tilde X_{n+1})$, $\hat k=\hat k(\tilde X_{n+1})$, and $k_\oracle=k_\oracle(\tilde X_{n+1})$. Further, define the empirical counterpart of $\mu_k$ by
\[
\hat\mu_k
:=
\frac1{n_{k,\train}}
\sum_{i=1}^{n_{k,\train}}
H(X_{i,k},\tilde X_{n+1}).
\]
By the definition of $\cG$ in \eqref{eq:choice_G_set},
\[
\Delta_{\cG,2}
=
\P_{P_{\test,\tilde X}\times\cD_\train}
\left(
\mu_{\hat k}\le\tfrac12\mu_{k_\oracle}
\right).
\]
Since $\hat k$ maximizes $\hat\mu_k$, we always have $\hat\mu_{\hat k}\ge\hat\mu_{k_\oracle}$. Hence,
\begin{align*}
\mu_{\hat k}-\mu_{k_\oracle}
&=
(\mu_{\hat k}-\hat\mu_{\hat k})
+
(\hat\mu_{\hat k}-\hat\mu_{k_\oracle})
+
(\hat\mu_{k_\oracle}-\mu_{k_\oracle})\\
&\ge
-2\max_{k\in[K]}|\hat\mu_k-\mu_k|.
\end{align*}
If $\mu_{\hat k}\le\mu_{k_\oracle}/2$, this inequality implies
$\max_{k\in[K]}|\hat\mu_k-\mu_k|\ge\mu_{k_\oracle}/4$.
Thus, for any $t>0$, it follows that
\begin{align*}
\Delta_{\cG,2}
&\le
\P(\mu_{k_\oracle}\le t)
+
\P\left(
\mu_{\hat k}\le\tfrac12\mu_{k_\oracle},
\ \mu_{k_\oracle}>t
\right)\\
&\le
\P(\mu_{k_\oracle}\le t)
+
\P\left(
\max_{k\in[K]}|\hat\mu_k-\mu_k|\ge t/4
\right).
\end{align*}
The first probability depends only on the perturbed test feature $\tilde{X}_{n+1}$, while the second one is over the randomness of both $\tilde{X}_{n+1}$ and the training data $\cD_{\train}$.

To bound the second probability, we condition on $\tilde X_{n+1}$. under this conditioning, the training observations remain independent, and each $H(X_{i,k},\tilde X_{n+1})$ has mean $\mu_k$ and lies in $[0,\|H\|_\infty]$ almost surely. Hoeffding's inequality therefore gives
\[
\P\left(
|\hat\mu_k-\mu_k|\ge t/4
\;\middle|\;
\tilde X_{n+1}
\right)
\le
2\exp\left(
-\frac{n_{k,\train}t^2}{8\|H\|_\infty^2}
\right).
\]
Taking a union bound over $k\in[K]$, using $n_{k,\train}\ge n_\eff$, and taking an expectation over $\tilde X_{n+1}$ yields
\[
\P\left(
\max_{k\in[K]}|\hat\mu_k-\mu_k|\ge t/4
\right)
\le
2K\exp\left(
-\frac{n_\eff t^2}{8\|H\|_\infty^2}
\right).
\]
For any $0<\delta<1$, choosing
$t=t_\delta=2\|H\|_\infty\sqrt{2\log(2K/\delta)/n_\eff}$
makes the right hand side of the aforementioned inequality to $\delta$. Consequently,
\[
\Delta_{\cG,2}
\le
\delta+
\P_{\tilde X_{n+1}\sim P_{\test,\tilde X}}
\left(
\mu_{k_\oracle(\tilde X_{n+1})}(\tilde X_{n+1})
\le t_\delta
\right).
\]
Taking the infimum over $0<\delta<1$ completes the proof.
\end{proof}

\begin{lemma}\label{lemma:Delta_G1_bdd}
Under the conditions of Theorem~\ref{thm:main_coverage},
\[
\Delta_{\cG,1}
\le
\E_{\tilde X_{n+1}\sim\bar P_{\tilde X}}
\left[
\inf_{\substack{A\subseteq\cX\\\epsilon>0}}
\frac{
2\left(
\frac{\|g_{\test,X}\|_\infty}{B}\omega_A
+\omega_{\test,A}
\right)\epsilon
+
2\|g_{\test,X}\|_\infty
\Delta(\tilde X_{n+1};A,\epsilon)
}{
\left(
1-\frac{2\omega_A\epsilon}{B}
-2\Delta(\tilde X_{n+1};A,\epsilon)
\right)\vee \frac{1}{K}
}
\right].
\]
\end{lemma}

\begin{proof}
We start by recalling that 
\[
\Delta_{\cG,1}:=
\E_{P_{\test,\tilde X}\times\cD_\train}
\left[
\TV\left(
P_{\test,X\mid\tilde X=\tilde X_{n+1}},
P_{\hat k,X\mid\tilde X=\tilde X_{n+1}}
\right)
\one_{\{\tilde X_{n+1}\in\cG\}}
\right].
\]
We first express the inner total variation distance in terms of the densities of source and test distributions with respect to the envelope distribution. Fix $k\in[K]$ and a perturbed feature $\tilde x\in \cG$ so that the relevant conditional distributions are well defined. By Bayes' formula,
\[
\frac{\d P_{\test,X\mid\tilde X=\tilde x}}{\d\bar P_X}(x)
=
\frac{g_{\test,X}(x)H(x,\tilde x)}
{\E_{P_{\test,X}}[H(X,\tilde x)]},
\qquad
\frac{\d P_{k,X\mid\tilde X=\tilde x}}{\d\bar P_X}(x)
=
\frac{g_{k,X}(x)H(x,\tilde x)}
{\E_{P_{k,X}}[H(X,\tilde x)]}.
\]
Consequently,
\begin{align*}
\TV\left(
P_{\test,X\mid\tilde X=\tilde x},
P_{k,X\mid\tilde X=\tilde x}
\right)=
\frac12\int
H(x,\tilde x)
\left|
\frac{g_{\test,X}(x)}
{\E_{P_{\test,X}}[H(X,\tilde x)]}
-
\frac{g_{k,X}(x)}
{\E_{P_{k,X}}[H(X,\tilde x)]}
\right|
\,\d\bar P_X(x).
\end{align*}
Let $Q_{\tilde x}:=\bar P_{X\mid\tilde X=\tilde x}$ and note that
\[
\frac{\d Q_{\tilde x}}{\d\bar P_X}(x)
=
\frac{H(x,\tilde x)}
{\E_{\bar P_X}[H(X,\tilde x)]}.
\]
Therefore, a change of measure gives
\begin{align}\label{eq:TV_iden_1}
\TV\left(
P_{\test,X\mid\tilde X=\tilde x},
P_{k,X\mid\tilde X=\tilde x}
\right)=
\frac{\E_{\bar P_X}[H(X,\tilde x)]}{2}
\E_{Q_{\tilde x}}
\left|
\frac{g_{\test,X}(X)}
{\E_{P_{\test,X}}[H(X,\tilde x)]}
-
\frac{g_{k,X}(X)}
{\E_{P_{k,X}}[H(X,\tilde x)]}
\right|.
\end{align}
For each source $k$, we have
\[
\E_{P_{k,X}}[H(X,\tilde x)]
=
\E_{\bar P_X}[g_{k,X}(X)H(X,\tilde x)]
=
\E_{\bar P_X}[H(X,\tilde x)]
\E_{Q_{\tilde x}}[g_{k,X}(X)],
\]
and the same identity holds for the test distribution. Let $X,X'$ be independent draws from $Q_{\tilde x}$. Applying the above identity, for fixed $X$,
\begin{align*}
&g_{\test,X}(X)\E_{P_{k,X}}[H(X,\tilde x)]
- g_{k,X}(X)\E_{P_{\test,X}}[H(X,\tilde x)]\\
&\quad=
\E_{\bar P_X}[H(X,\tilde x)]
\E_{X'\sim Q_{\tilde x}}
\left[
g_{\test,X}(X)g_{k,X}(X')
-
g_{k,X}(X)g_{\test,X}(X')
\right].
\end{align*}
Taking absolute values and applying Jensen's inequality to the inner expectation over $X'$, followed by \eqref{eq:TV_iden_1}, yields
\begin{align}\label{eq:TV_bdd_1}
\TV\left(
P_{\test,X\mid\tilde X=\tilde x},
P_{k,X\mid\tilde X=\tilde x}
\right)\le
\frac{\E_{\bar P_X}[H(X,\tilde x)]^2}
{2\E_{P_{\test,X}}[H(X,\tilde x)]
\E_{P_{k,X}}[H(X,\tilde x)]}
T_k(\tilde x),
\end{align}
where we define
\begin{equation}\label{eq:def_T_k}
T_k(\tilde x)
:=
\E_{X,X'\stackrel{iid}{\sim}Q_{\tilde x}}
\left|
g_{\test,X}(X)g_{k,X}(X')
-
g_{k,X}(X)g_{\test,X}(X')
\right|.
\end{equation}

We now apply this bound to $\Delta_{\cG,1}$. The perturbed test and envelope distributions have densities
$\E_{P_{\test,X}}[H(X,\tilde x)]$ and
$\E_{\bar P_X}[H(X,\tilde x)]$, respectively, with respect to $\mu$. Thus, by a change of measure gives
\begin{align*}
\Delta_{\cG,1}=\E_{\bar P_{\tilde X}\times\cD_\train}
\bigg[\TV\left(
P_{\test,X\mid\tilde X=\tilde X_{n+1}},
P_{\hat k,X\mid\tilde X=\tilde X_{n+1}}
\right)
\one_{\{\tilde X_{n+1}\in\cG\}}\times
\frac{\E_{P_{\test,X}}[H(X,\tilde X_{n+1})]}
{\E_{\bar P_X}[H(X,\tilde X_{n+1})]}
\bigg].
\end{align*}
Substituting \eqref{eq:TV_bdd_1}, we obtain
\[
\Delta_{\cG,1}
\le
\E_{\bar P_{\tilde X}\times\cD_\train}
\left[
\frac{\E_{\bar P_X}[H(X,\tilde X_{n+1})]}
{2\E_{P_{\hat k,X}}[H(X,\tilde X_{n+1})]}
T_{\hat k}(\tilde X_{n+1})
\one_{\{\tilde X_{n+1}\in\cG\}}
\right].
\]
The term inside the expectation is taken to be zero outside $\cG$. 

By definition of $\cG$ in \eqref{eq:choice_G_set} and then dropping the indicator, we conclude that
\begin{equation}\label{eq:bdd_Delta_G1}
\Delta_{\cG,1}
\le
\E_{\bar P_{\tilde X}\times\cD_\train}
\left[
T(\tilde X_{n+1})T_{\hat k}(\tilde X_{n+1})
\right],
\end{equation}
where we define
\begin{equation}\label{eq:T_tilde_X_n1}
T(\tilde x)
:=
\frac{\E_{\bar P_X}[H(X,\tilde x)]}
{\E_{P_{k_\oracle(\tilde x),X}}[H(X,\tilde x)]}.
\end{equation}
Note that the denominator is positive for $\bar P_{\tilde X}$-almost every $\tilde x$: since $\sum_{k=1}^K g_{k,X}\ge B$ almost everywhere, at least one source has positive local mass. We complete the proof by deriving upper bounds on $T_k(\tilde x)$ and $1/T(\tilde x)$.

\paragraph{Controlling $T_{\hat k}(\tilde X_{n+1})$.}
Fix a measurable set $A\subseteq\cX$ and $\epsilon>0$. For $x,x'\in A$ with $\|x-x'\|_2\le2\epsilon$, the triangle inequality gives
\begin{align*}
&\left|
g_{\test,X}(x)g_{k,X}(x')
-
g_{k,X}(x)g_{\test,X}(x')
\right|\\
&\quad\le
g_{\test,X}(x)|g_{k,X}(x')-g_{k,X}(x)|
+
g_{k,X}(x)|g_{\test,X}(x)-g_{\test,X}(x')|\\
&\quad\le
\left(
\|g_{\test,X}\|_\infty\omega_{k,A}
+
B\omega_{\test,A}
\right)\|x-x'\|_2\\
&\quad\le
2\left(
\|g_{\test,X}\|_\infty\omega_A
+
B\omega_{\test,A}
\right)\epsilon.
\end{align*}
Here, we used $g_{k,X}\le B$ and
$g_{\test,X}\le\|g_{\test,X}\|_\infty$.
On the other hand, for arbitrary $x,x'$, the absolute difference is at most
$B\|g_{\test,X}\|_\infty$, since both products lie between zero and this value.

For independent $X,X'\stackrel{iid}{\sim} Q_{\tilde x}$, a union bound and the triangle inequality give
\begin{align*}
&\P_{X,X'\stackrel{iid}{\sim} Q_{\tilde{x}}}
\left(
X\notin A
\text{ or }X'\notin A
\text{ or }\|X-X'\|_2>2\epsilon
\right)\\
&\quad\le
2Q_{\tilde x}(A^c)
+
2Q_{\tilde x}\bigl(\|X-\tilde x\|_2>\epsilon\bigr)=
2\Delta(\tilde x;A,\epsilon).
\end{align*}
Applying the local bound when both points belong to $A$ and are within distance $2\epsilon$, and the uniform bound otherwise, yields
\begin{equation}\label{eq:T_k_final_bdd}
T_k(\tilde x)
\le
2\left(
\|g_{\test,X}\|_\infty\omega_A
+
B\omega_{\test,A}
\right)\epsilon
+
2B\|g_{\test,X}\|_\infty
\Delta(\tilde x;A,\epsilon).
\end{equation}
This bound holds for every $k\in[K]$.

\paragraph{Controlling $T(\tilde{X}_{n+1})$.}
Let $\bar k(x)\in\arg\max_{k\in[K]}g_{k,X}(x)$, with a fixed rule for resolving ties, and draw $X'\sim Q_{\tilde x}$ independent of all the data. Conditional on $X'$, the index $\bar k(X')$ is fixed. By the definition of the oracle choice,
\[
\E_{P_{k_\oracle(\tilde x),X}}[H(X,\tilde x)]
\ge
\E_{P_{\bar k(X'),X}}[H(X,\tilde x)].
\]
Applying the above inequality, and then change of measures to $\bar P_X$ and then to $Q_{\tilde x}$, we obtain
\begin{align*}
\frac1{T(\tilde x)}
\ge 
\frac{\E_{P_{\bar k(X'),X}}[H(X,\tilde x)]}{\E_{\bar P_X}[H(X,\tilde x)]}&=
\frac{
\E_{\bar P_X}[g_{\bar k(X'),X}(X)H(X,\tilde x)]
}{
\E_{\bar P_X}[H(X,\tilde x)]
}\\
&=
\E_{X\sim Q_{\tilde x}}[g_{\bar k(X'),X}(X)]\\
&\ge
g_{\bar k(X'),X}(X')
-
\E_{X\sim Q_{\tilde x}}
\left[
|g_{\bar k(X'),X}(X)-g_{\bar k(X'),X}(X')|
\right].
\end{align*}
By Lemma~\ref{lemma:bdd_g_k},
$g_{\bar k(X'),X}(X')=B$ almost surely. Since the left-hand side does not depend on $X'$, taking expectation over $X'$ gives
\[
\frac1{T(\tilde x)}
\ge
B-
\E_{X,X'\stackrel{iid}{\sim} Q_{\tilde{x}}}
\left[
|g_{\bar k(X'),X}(X)-g_{\bar k(X'),X}(X')|
\right].
\]

For $x,x'\in A$ with $\|x-x'\|_2\le2\epsilon$, the absolute difference inside this expectation is at most $2\omega_A\epsilon$. This holds even though the source index depends on $x'$, because $\omega_A$ bounds the local Lipschitz seminorm of every source density. For arbitrary $x,x'$, the difference is at most $B$. Applying the same union bound as in the argument of bounding the numerator, we obtain
\[
\E_{X,X'\stackrel{iid}{\sim} Q_{\tilde{x}}}
\left[
|g_{\bar k(X'),X}(X)-g_{\bar k(X'),X}(X')|
\right]
\le
2\omega_A\epsilon
+
2B\Delta(\tilde x;A,\epsilon).
\]
Consequently,
\[
\frac1{T(\tilde x)}
\ge
\left(
B-2\omega_A\epsilon
-2B\Delta(\tilde x;A,\epsilon)
\right)\vee0.
\]

Further, since $\max_{k\in[K]}g_{k,X}=B$ almost everywhere, we have $\sum_{k=1}^K g_{k,X}\ge B$. Consequently, we obtain
\begin{align*}
\frac1{T(\tilde x)}
=\frac{\max_{k\in[K]}\E_{\bar P_X}[g_{k,X}(X)H(X,\tilde x)]}{\E_{\bar P_X}[H(X,\tilde x)]}\ge
\frac{
\E_{\bar P_X}\left[\sum_{k=1}^K g_{k,X}(X)H(X,\tilde x)\right]
}{
K\E_{\bar P_X}[H(X,\tilde x)]
}
\ge \frac{B}{K}.
\end{align*}
Combining this with the preceding lower bound gives
\begin{equation}\label{eq:lb_ratio_T}
\frac1{T(\tilde x)}
\ge
\left(
B-2\omega_A\epsilon
-2B\Delta(\tilde x;A,\epsilon)
\right)\vee\frac{B}{K}.
\end{equation}

Combining \eqref{eq:T_k_final_bdd} and \eqref{eq:lb_ratio_T}, and dividing the numerator and denominator by $B$, gives
\[
T(\tilde x)T_k(\tilde x)
\le
\frac{
2\left(
\frac{\|g_{\test,X}\|_\infty}{B}\omega_A
+
\omega_{\test,A}
\right)\epsilon
+
2\|g_{\test,X}\|_\infty
\Delta(\tilde x;A,\epsilon)
}{
\left(
1-\frac{2\omega_A\epsilon}{B}
-2\Delta(\tilde x;A,\epsilon)
\right)\vee \frac{1}{K}
}
\]
for every $k\in[K]$. Taking the infimum over $A$ and $\epsilon$ gives a bound that is independent of the selected source or the training data. Substituting this bound into \eqref{eq:bdd_Delta_G1} and taking expectation over $\tilde X_{n+1}\sim\bar P_{\tilde X}$ completes the proof.
\end{proof}

\begin{lemma}\label{lemma:Delta_G1_bdd_lipschitz}
Under the conditions of Theorem~\ref{thm:main_coverage},
suppose additionally that $g_{1,X},\ldots,g_{K,X}$ and
$g_{\test,X}$ are globally $L$-Lipschitz, and that $\E_{\substack{X\sim\bar P_X\\   \tilde X\mid X\sim H(X,\cdot)}}
\|X-\tilde X\|_2<\infty$. Then,
\[
\Delta_{\cG,1}
\le
2L\left(1+\frac{2\|g_{\test,X}\|_\infty}{B}\right)
\E_{\substack{X\sim\bar P_X\\
              \tilde X\mid X\sim H(X,\cdot)}}
\|X-\tilde X\|_2.
\]
\end{lemma}

\begin{proof}
We revisit the proof of Lemma~\ref{lemma:Delta_G1_bdd}.
Under global Lipschitzness, we can simplify the arguments and the bound. Recall that $Q_{\tilde x}:=\bar P_{X\mid\tilde X=\tilde x}$, and $T(\tilde{x})$ is as defined in~\eqref{eq:T_tilde_X_n1}.

Further, define
\[
\ell(\tilde x):=
\frac{\E_{P_{\test,X}}[H(X,\tilde X_{n+1})]}
{\E_{\bar P_X}[H(X,\tilde X_{n+1})]}
\E_{\cD_\train}
\left[
\TV\left(
P_{\test,X\mid\tilde X=\tilde x},
P_{\hat k,X\mid\tilde X=\tilde x}
\right)
\one_{\{\tilde x\in\cG\}}
\right],
\]
where $\hat k=\hat k(\tilde x)$. The integrand is taken
to be zero outside $\cG$.
By construction, $\Delta_{\cG,1}
=
\E_{\tilde X\sim\bar P_{\tilde X}}[\ell(\tilde X)]$.
Moreover, we have
\[
\frac{\E_{P_{\test,X}}[H(X,\tilde X_{n+1})]}
{\E_{\bar P_X}[H(X,\tilde X_{n+1})]}
= \E_{Q_{\tilde x}}[g_{\test,X}(X)]
\le \|g_{\text{test},X}\|_\infty,
\]
and the inner total variation term is at most one, so that we have
$0\le\ell(\tilde x)\le \|g_{\text{test},X}\|_\infty$.
From the proof of Lemma~\ref{lemma:Delta_G1_bdd}, we note that by~\eqref{eq:TV_bdd_1}, together with the definition of $\cG$ in~\eqref{eq:choice_G_set}, gives that~\eqref{eq:bdd_Delta_G1}:
\[
\ell(\tilde x)
\le
T(\tilde x)
\E_{\cD_\train}
\left[
T_{\hat k}(\tilde x)
\one_{\{\tilde x\in\cG\}}
\right].
\]
Here $T_k$ is defined in~\eqref{eq:def_T_k}.

By Lipschitzness of the densities and the bounds
$g_{k,X}\le B$ and $g_{\test,X}\le \|g_{\test,X}\|_\infty$, we have that
\begin{align*}
T_k(\tilde x)
&\le
\E_{X,X'\stackrel{iid}{\sim}Q_{\tilde x}}
\Bigl[
g_{\test,X}(X)
\bigl|g_{k,X}(X')-g_{k,X}(X)\bigr|+
g_{k,X}(X)
\bigl|g_{\test,X}(X)-g_{\test,X}(X')\bigr|
\Bigr]
\\
&\le L(\|g_{\text{test},X}\|_\infty+B)\,\E_{X,X'\stackrel{iid}{\sim}Q_{\tilde x}}
\|X-X'\|_2.
\end{align*}
Consequently, it follows that
\begin{equation}\label{eq:lipschitz_weighted_tv_bound}
\frac{\ell(\tilde{x})}{T(\tilde{x})}=\E_{\cD_\train}
\left[
T_{\hat k}(\tilde x)
\one_{\{\tilde x\in\cG\}}
\right]
\le L(\|g_{\text{test},X}\|_\infty+B)\cdot\E_{X,X'\stackrel{iid}{\sim}Q_{\tilde x}}
\|X-X'\|_2.
\end{equation}

Next, we may similarly repeat the calculation for proving lower bound on $1/T(\tilde{x})$,
preceding~\eqref{eq:lb_ratio_T} in the proof of Lemma~\ref{lemma:Delta_G1_bdd}.
In particular, we derive
\begin{align*}
\frac1{T(\tilde x)}
&\ge
B-
\E_{X,X'\stackrel{iid}{\sim}Q_{\tilde x}}
\left|
g_{\bar k(X'),X}(X)
-
g_{\bar k(X'),X}(X')
\right|
\\
&\ge B-L\,\E_{X,X'\stackrel{iid}{\sim}Q_{\tilde x}}
\|X-X'\|_2.
\end{align*}
Combining the above inequality with
\eqref{eq:lipschitz_weighted_tv_bound} and further recalling
$\ell(\tilde x)\le \|g_{\text{test},X}\|_\infty$ yields
\begin{align*}
B\ell(\tilde x)
&\le
\left\{
\frac1{T(\tilde x)}+L\E_{X,X'\stackrel{iid}{\sim}Q_{\tilde x}}
\|X-X'\|_2.
\right\}\ell(\tilde x)
\\
&\le
L(\|g_{\text{test},X}\|_\infty+B)\E_{X,X'\stackrel{iid}{\sim}Q_{\tilde x}}
\|X-X'\|_2.+L\|g_{\text{test},X}\|_\infty \E_{X,X'\stackrel{iid}{\sim}Q_{\tilde x}}
\|X-X'\|_2.\\
&=
L(B+2\|g_{\text{test},X}\|_\infty)\E_{X,X'\stackrel{iid}{\sim}Q_{\tilde x}}
\|X-X'\|_2..
\end{align*}
Finally, taking expectations, we obtain
\[
\Delta_{\cG,1}
\le
L\left(1+\frac{2\|g_{\text{test},X}\|_\infty}{B}\right)
\E_{\tilde X\sim\bar P_{\tilde X}}[\E_{X,X'\stackrel{iid}{\sim}Q_{\tilde x}}
\|X-X'\|_2].
\]
Finally, the triangle inequality and the tower property give
\begin{align*}
\E_{\tilde X\sim\bar P_{\tilde X}}[\E_{X,X'\stackrel{iid}{\sim}Q_{\tilde x}}
\|X-X'\|_2]
\le
2\E_{\tilde X\sim\bar P_{\tilde X}}
\E_{X\sim Q_{\tilde X}}
\|X-\tilde X\|_2=
2\E_{\substack{X\sim\bar P_X\\
               \tilde X\mid X\sim H(X,\cdot)}}
\|X-\tilde X\|_2.
\end{align*}
Substituting this into the preceding bound proves the result.
\end{proof}

\section{Test-Conditional Coverage of \texttt{MS-RLCP}}
\label{app:cond_cov}

In this section, we study the coverage of \texttt{MS-RLCP} conditional on a fixed test feature $X_{n+1}=x_0$, namely,
\begin{align*}
    \P\left(Y_{n+1}\in \hat C_n(X_{n+1})\mid X_{n+1}=x_0\right).
\end{align*}
As discussed earlier, exact distribution-free conditional coverage is generally impossible with finite prediction sets \citep{vovk2012conditional,foygel2021limits}. A large body of literature therefore imposes additional distributional assumptions and seeks such guarantees in the asymptotic regime. We follow the same approach and develop analogous results for \texttt{MS-RLCP} in the heterogeneous multi-source setting.

We adopt the multi-source setting from Section~\ref{sec:msrlcp_theory_main}. We assume balanced source datasets, with each $\cD_k$ containing $2n$ observations, split equally between $\cD_{k,\train}$ and $\cD_{k,\cal}$. To simplify the exposition, throughout this section we take $\cX=\R^d$, with $d\ge1$, and use the box kernel in \eqref{eq:gauss_box_kernels} with a deterministic bandwidth $h_n>0$ at total sample size $2nK$.

We further impose the following conditions on the source distributions and learned scores:
\begin{enumerate}[label=(A\arabic*), ref=A\arabic*]
    \item\label{A1} Each source feature distribution $P_{k,X}$ admits a Lebesgue density $f_{k,X}$ satisfying $f_{k,X}>0$ almost everywhere.

    \item\label{A2} For each $k\in[K]$, conditional on $\cD_{k,\train}$, the score $s_k(X,Y)$ has a non-atomic distribution under an independent draw $(Y,X)\sim P_k$, almost surely with respect to the training data.
\end{enumerate}

As before, write $\cD_\train=\bigcup_{k\in[K]}\cD_{k,\train}$. For probability measures $P,Q$ on $\R$, we define the Kolmogorov distance between $P$ and $Q$ by
\[
d_{\mathrm{Kol}}(P,Q):=\sup_{t\in\R}
\left|P((-\infty,t])-Q((-\infty,t])\right|.
\]
Further, for a feature value $x$, let
\[
Q_{k,n}(x)
:=
\mathcal L\bigl(s_k(X,Y)\mid\cD_\train,X=x\bigr),
\qquad
Y\sim P_{Y\mid X}(\cdot\mid x),
\]
denote the score distribution at the $k$th source, conditional on the feature value $x$ and the training data $\cD_\train$. The subscript $n$ records its dependence on the training sample size.

\begin{theorem}\label{thm:cond_cov_finite_sample}
Under the above setting and Assumptions~(\ref{A1})--(\ref{A2}), for every fixed $x_0\in\R^d$,
\begin{align}\label{eq:cond_cov_finite_sample_bound}
&\left|
\P\left(
Y_{n+1}\in\hat C_n(X_{n+1})
\;\middle|\;
X_{n+1}=x_0
\right)
-(1-\alpha)
\right|
\nonumber\\
&\quad\le
\E_{\tilde X\sim H_n(x_0,\cdot)}
\left[
2K\exp\left(
-\frac{n}{36}\max_{k\in[K]}P_{k,X}(B_{h_n}(\tilde X))
\right)
+
\frac{2}{
(n+1)\max_{k\in[K]}P_{k,X}(B_{h_n}(\tilde X))
}
\right]
\nonumber\\
&\hspace{2cm}+
\E_{\cD_\train}
\left[
\max_{k\in[K]}
\sup_{x\in B_{2h_n}(x_0)}
d_{\mathrm{Kol}}\bigl(Q_{k,n}(x),Q_{k,n}(x_0)\bigr)
\right].
\end{align}
\end{theorem}

\begin{proof}
Conditional on $\tilde X_{n+1}$ and $\cD_\train$, the selected source $\hat k=\hat k(\tilde X_{n+1})$ is fixed. Index the calibration observations from this source by $i\in[n]$ and write
\[
S_i:=s_{\hat k}(X_{i,\hat k},Y_{i,\hat k}),
\qquad
S_{n+1}:=s_{\hat k}(X_{n+1},Y_{n+1}).
\]
Further, let
\[
I(\tilde X_{n+1})
:=
\{i\in[n]:X_{i,\hat k}\in B_{h_n}(\tilde X_{n+1})\},
\qquad
N_n:=|I(\tilde X_{n+1})|.
\]
The box kernel assigns equal weights to the calibration observations in $I(\tilde X_{n+1})$ and the test point. Thus, the prediction threshold reduces to
\[
\hat q_{1-\alpha}
=
\operatorname{Quantile}_{1-\alpha}
\bigl(\{S_i:i\in I(\tilde X_{n+1})\},+\infty\bigr),
\]
and the \texttt{MS-RLCP} prediction set covers the test response if and only if $S_{n+1}\le\hat q_{1-\alpha}$.

Now fix the training data $\cD_\train$, the test feature $X_{n+1}=x_0$, the perturbed test feature $\tilde X_{n+1}=\tilde x$, and the index set $I(\tilde x)$. The scores corresponding to the indices in $I(\tilde x)$ are independent draws from
\[
Q^{\mathrm{loc}}_{\hat k,n}(\tilde x)
:=
\int_{B_{h_n}(\tilde x)}
Q_{\hat k,n}(x)\,
\d P_{\hat k,X\mid X\in B_{h_n}(\tilde x)}(x).
\]
The conditioning event $\{X\in B_{h_n}(\tilde x)\}$ has positive probability by Assumption~(\ref{A1}). By Assumption~(\ref{A2}), the resulting conditional score distribution $Q^{\mathrm{loc}}_{\hat k,n}(\tilde x)$ is non-atomic.

Under the same conditioning on the training data, test feature, perturbation, and index set, the test score is independent of the calibration scores and has distribution $Q_{\hat k,n}(x_0)$.

Consider an independent draw $S'_{n+1}\sim Q^{\mathrm{loc}}_{\hat k,n}(\tilde x)$. Since we have conditioned on $I(\tilde x)$, $N_n$ is also fixed. The resulting $N_n+1$ scores are i.i.d. draws from a non-atomic distribution. The standard conformal rank argument \citep{vovk2005algorithmic,shafer2008tutorial} therefore gives
\[
\P\left(
S'_{n+1}\le\hat q_{1-\alpha}
\;\middle|\;
I(\tilde x),\tilde X_{n+1}=\tilde x,
X_{n+1}=x_0,\cD_\train
\right)
=
\frac{\lceil(N_n+1)(1-\alpha)\rceil}{N_n+1},
\]
which differs from $1-\alpha$ by at most $1/(N_n+1)$.

Returning to the original coverage event $\{S_{n+1}\le\hat q_{1-\alpha}\}$, the change in its probability is bounded by the Kolmogorov distance between the two score distributions. In particular,
\begin{align*}
&\left|
\P\left(
S_{n+1}\le\hat q_{1-\alpha}
\;\middle|\;
I(\tilde x),\tilde X_{n+1}=\tilde x,
X_{n+1}=x_0,\cD_\train
\right)
-(1-\alpha)
\right|\\
&\qquad\le
\frac{1}{N_n+1}
+
d_{\mathrm{Kol}}\bigl(
Q^{\mathrm{loc}}_{\hat k,n}(\tilde x),
Q_{\hat k,n}(x_0)
\bigr).
\end{align*}

By convexity of total variation under mixing,
\begin{align*}
\TV\bigl(
Q^{\mathrm{loc}}_{\hat k,n}(\tilde x),
Q_{\hat k,n}(x_0)
\bigr)
&\le
\int_{B_{h_n}(\tilde x)}
d_{\mathrm{Kol}}\bigl(Q_{\hat k,n}(x),Q_{\hat k,n}(x_0)\bigr)
\,\d P_{\hat k,X\mid X\in B_{h_n}(\tilde x)}(x)\\
&\le
\max_{k\in[K]}
\sup_{x\in B_{2h_n}(x_0)}
d_{\mathrm{Kol}}\bigl(Q_{k,n}(x),Q_{k,n}(x_0)\bigr),
\end{align*}
where the last inequality follows from
$\tilde x\in B_{h_n}(x_0)$ and
$B_{h_n}(\tilde x)\subseteq B_{2h_n}(x_0)$.

Taking expectations over $\cD_\train$, $\tilde X_{n+1}$, and $I(\tilde X_{n+1})$ therefore yields
\begin{align}\label{eq:conditional_score_decomposition}
&\left|
\P\left(
Y_{n+1}\in\hat C_n(X_{n+1})
\;\middle|\;
X_{n+1}=x_0
\right)
-(1-\alpha)
\right|
\nonumber\\
&\quad\le
\E\left[\frac{1}{N_n+1}\;\middle|\;X_{n+1}=x_0\right]
+
\E_{\cD_\train}
\left[
\max_{k\in[K]}
\sup_{x\in B_{2h_n}(x_0)}
d_{\mathrm{Kol}}\bigl(Q_{k,n}(x),Q_{k,n}(x_0)\bigr)
\right].
\end{align}

It remains to bound the first error term. Fix $\tilde X_{n+1}=\tilde x$ and $X_{n+1}=x_0$. For the box kernel, the source-selection rule satisfies
\[
\hat k=\hat k(\tilde x)
\in
\arg\max_{k\in[K]}
\frac1n\sum_{i=1}^n
\one\{X_{i,k}^{\train}\in B_{h_n}(\tilde x)\}.
\]
Following the argument in Lemma~\ref{lemma:Delta_2_bdd}, we obtain
\begin{align*}
P_{\hat k,X}(B_{h_n}(\tilde x))
\ge{}&
\max_{k\in[K]}P_{k,X}(B_{h_n}(\tilde x))\\
&-
2\max_{k\in[K]}
\left|
\frac1n\sum_{i=1}^n
\one\{X_{i,k}^{\train}\in B_{h_n}(\tilde x)\}
-
P_{k,X}(B_{h_n}(\tilde x))
\right|.
\end{align*}
The indicators are independent Bernoulli variables. Each
$\one\{X_{i,k}^{\train}\in B_{h_n}(\tilde x)\}$
has mean $P_{k,X}(B_{h_n}(\tilde x))$ and variance at most
$\max_{j\in[K]}P_{j,X}(B_{h_n}(\tilde x))$.
Applying Bernstein's inequality to the $k$th source gives
\begin{align*}
&\P_{\cD_\train}
\left(
\left|
\frac1n\sum_{i=1}^n
\one\{X_{i,k}^{\train}\in B_{h_n}(\tilde x)\}
-
P_{k,X}(B_{h_n}(\tilde x))
\right|
\ge
\frac14\max_{j\in[K]}P_{j,X}(B_{h_n}(\tilde x))
\right)\\
&\qquad\le
2\exp\left(
-\frac{n}{36}
\max_{j\in[K]}P_{j,X}(B_{h_n}(\tilde x))
\right).
\end{align*}
A union bound then yields
\begin{align*}
\P_{\cD_\train}
\left(
P_{\hat k,X}(B_{h_n}(\tilde x))
\le
\frac12\max_{k\in[K]}P_{k,X}(B_{h_n}(\tilde x))
\right)\le
2K\exp\left(
-\frac{n}{36}\max_{k\in[K]}P_{k,X}(B_{h_n}(\tilde x))
\right).
\end{align*}

Conditional on $\cD_\train$ and $\tilde X_{n+1}=\tilde x$,
$N_n\sim\operatorname{Binomial}
\bigl(n,P_{\hat k,X}(B_{h_n}(\tilde x))\bigr)$.
Moreover, for $N\sim\operatorname{Binomial}(n,p)$ with $p>0$,
\[
\E\left[\frac1{N+1}\right]
\le
\frac1{(n+1)p}.
\]
On the event
$\{P_{\hat k,X}(B_{h_n}(\tilde x))
>
\frac12\max_{k\in[K]}P_{k,X}(B_{h_n}(\tilde x))\}$,
this inequality gives
\[
\E\left[
\frac{1}{N_n+1}
\;\middle|\;
X_{n+1}=x_0,\tilde X_{n+1}=\tilde x,\cD_\train
\right]
\le
\frac{2}{
(n+1)\max_{k\in[K]}P_{k,X}(B_{h_n}(\tilde x))
}.
\]
On the complementary event, we use $1/(N_n+1)\le1$. Averaging over the training data, we obtain
\begin{align*}
&\E\left[
\frac1{N_n+1}
\;\middle|\;
\tilde X_{n+1}=\tilde x,X_{n+1}=x_0
\right]\\
&\quad\le
2K\exp\left(
-\frac{n}{36}\max_{k\in[K]}P_{k,X}(B_{h_n}(\tilde x))
\right)
+
\frac{2}{
(n+1)\max_{k\in[K]}P_{k,X}(B_{h_n}(\tilde x))
}.
\end{align*}
Taking an expectation over $\tilde X_{n+1}$ and substituting into
\eqref{eq:conditional_score_decomposition} completes the proof.
\end{proof}

To establish asymptotic conditional coverage, we impose the following additional conditions:
\begin{enumerate}[label=(A\arabic*), ref=A\arabic*, start=3]
    \item\label{A3} The number of sources $K$ is fixed, $h_n\to0$, and $nh_n^d\to\infty$.

    \item\label{A4} The envelope distribution admits a Lebesgue density $\bar f_X$ that is bounded away from zero in a neighborhood of $x_0$.

    \item\label{A5} The learned score distributions satisfy
    \begin{equation}\label{eq:cont_tv}
    \max_{k\in[K]}
    \sup_{x\in B_{2h_n}(x_0)}
    d_{\mathrm{Kol}}\bigl(Q_{k,n}(x),Q_{k,n}(x_0)\bigr)
    \xrightarrow{\P}0,
    \end{equation}
    where convergence is with respect to the randomness of the training data.
\end{enumerate}
The complete proof of asymptotic test-conditional coverage is given in Appendix~\ref{app:proof_of_test_cond_cov}.

\section{Additional details on numerical experiments}
In this section, we present additional details about the experiments in Section~\ref{sec:real_experiment}. In the simulations in Section~\ref{sec:sims}, we used a per-source sample size of $5000$ and a train-calibration split of $0.5$. The test sample size was $1000$, and results are reported over $100$ repetitions.

\subsection{Additional details for the FMoW experiment}
\label{app:fmow_details}

We use the 2016 FMoW \citep{christie2018functional} slice and treat the geographical regions Africa, the Americas, Asia, Europe, and Oceania as individual sources. For each region, the initial 40\% split is used only for the shared backbone. The remaining 60\% is split into $5/8$ training, $2/8$ calibration, and $1/8$ test data, corresponding to 37.5\%, 15\%, and 7.5\% of the full regional slice. Images are resized to $224\times224$ and normalized using ImageNet statistics. The DenseNet--121 backbone is fine-tuned for 30 epochs with AdamW, a learning rate of $10^{-4}$, weight decay of $10^{-4}$, and a batch size of 32. After the backbone is frozen, each server trains a two-layer classifier head with ReLU and dropout 0.2 for 50 epochs using Adam with a learning rate of $10^{-3}$, weight decay of $10^{-4}$, and a batch size of 128. For this experiment, we used one NVIDIA RTX A6000 GPU.

Let $e(x)$ be the normalized DenseNet feature. We fit $\mathrm{PCA}_{16}$ on pooled server-training features and use $z(x)=\mathrm{PCA}_{16}(e(x))$ only for routing and calibration. For each test set, we use the Gaussian kernel in \eqref{eq:gauss_box_kernels} to generate perturbed versions of $z(\cdot)$. The kernel bandwidth $h$ is chosen using the median-distance heuristic over server calibration features and multiplied by 0.35. We use RAPS scores with $k_{\rm reg}=5$ and $\lambda=0.01$, at miscoverage level $\alpha=0.1$. We repeat the downstream split, head training, calibration, and evaluation for $50$ repetitions, using a fixed backbone feature cache and up to 500 test examples per region per repetition.

\subsection{Additional details for the MEPS experiment}
\label{app:meps-details}

For each MEPS panel, we use the public regression-form files and retain only continuous covariates. One-hot categorical columns are removed from the features. Negative continuous entries are treated as missing-value sentinels, and rows with missing covariates, invalid race labels (the sensitive attribute/source indicator), or invalid utilization values are discarded. We apply a log transformation to both the features and the response. In each repetition, White and Non-White examples are split separately into \(60\%\) training, \(20\%\) calibration, and \(20\%\) test sets. Feature standardization is fitted using the pooled training data and then applied to all splits.

Each source \(k\) fits a heteroskedastic Gaussian model:
\begin{align}\label{eq:het_gauss}
    Y\mid X=x,k\sim \mathcal{N}(\hat\mu_k(x),\hat\sigma_k^2(x)).
\end{align}
The mean model is a gradient-boosted regressor. The variance model is a second gradient-boosted regressor trained on five-fold out-of-fold squared residuals. Both boosting models use 300 estimators, a learning rate of \(0.05\), maximum depth \(3\), minimum leaf size \(10\), and a subsampling rate of \(0.8\). For source $k$, the calibration scores are
\[
S_{i,k}=\frac{|Y_i-\hat\mu_k(X_i)|}{\hat\sigma_k(X_i)}.
\]

When the bandwidth is not fixed manually, \(h\) is chosen as the median of the server-wise median pairwise distances among at most 256 calibration covariates, with a lower bound of \(10^{-3}\).

To perturb a test feature $x$, we sample
\(\tilde x=x+\varepsilon\), where \(\varepsilon\sim\mathcal{N}(0,h^2I)\) and $h$ is chosen as described above.

At the selected source, we compute the weighted \(0.9\)-quantile of the calibration scores, augmented with an infinite test-point pseudo-score. Calibration weights are
\(\exp(-\|X_i-\tilde x\|_2^2/(2h^2))\), and the test-point weight is
\(\exp(-\|x-\tilde x\|_2^2/(2h^2))\). The final interval is
\[
\widehat C(x)=
\left[
\max\{0,\hat\mu_{\hat k}(x)-\tau(x)\hat\sigma_{\hat k}(x)\},
\hat\mu_{\hat k}(x)+\tau(x)\hat\sigma_{\hat k}(x)
\right].
\]
We repeat the full pipeline over 50 repetitions for each panel and report means and standard deviations across seeds.
\end{document}